\documentclass[letterpaper]{article} 
\usepackage{aaai24}  
\usepackage{times}  
\usepackage{helvet}  
\usepackage{courier}  
\usepackage[hyphens]{url}  
\usepackage{graphicx} 
\usepackage{natbib}  
\usepackage{caption} 
\usepackage{algorithm}
\usepackage{algorithmic}

\usepackage{newfloat}
\usepackage{listings}
\DeclareCaptionStyle{ruled}{labelfont=normalfont,labelsep=colon,strut=off} 
\floatstyle{ruled}
\newfloat{listing}{tb}{lst}{}
\floatname{listing}{Listing}
\usepackage{amsmath}
\usepackage{amssymb}
\usepackage{xspace}
\usepackage{amsthm}
\usepackage{subcaption}
\usepackage{multirow}
\usepackage{booktabs}
\usepackage{xcolor}
\usepackage{import}
\usepackage{enumitem}
\usepackage{cancel}
\usepackage{mathtools}
\usepackage{tabularx}
\usepackage{ifthen}

\allowdisplaybreaks

\newcounter{arxiv}
\DeclareMathOperator*{\argmax}{arg\,max}
\DeclareMathOperator*{\esssup}{ess\,sup}

\newcommand{\bydef}{\triangleq}
\newcommand{\probd}{\mathbb{P}}

\newcommand{\OpProbFlat}[1]{\probd( {#1} )}

\newcommand{\dif}{\mathop{}\!\mathrm{d}}

\newcommand{\Estim}[1]{\hat{#1}}
\newcommand{\RewardEstim}[1]{\tilde{#1}}
\newcommand{\Rmax}[1]{R^{\max}_{#1}}
\newcommand{\Vmax}[1]{V^{\max}_{#1}}
\newcommand{\absval}[1]{\left|{#1}\right|}
\newcommand{\absvalflat}[1]{\lvert{#1}\rvert}
\newcommand{\CondCurlyBrackets}[1]{\expandafter\ifx\expandafter\relax\detokenize{#1}\relax\else\left(#1\right)\fi}
\newcommand{\CondRectBrackets}[1]{\expandafter\ifx\expandafter\relax\detokenize{#1}\relax\else[#1]\fi}
\newcommand{\CondRectBracketsHigh}[1]{\expandafter\ifx\expandafter\relax\detokenize{#1}\relax\else\left[#1\right]\fi}
\newcommand{\CondColon}[1]{\expandafter\ifx\expandafter\relax\detokenize{#1}\relax\else:{#1}\fi}

\newcommand{\Expt}{\mathbb{E}}

\newcommand{\ExptFlat}[3]{\ensuremath{\Expt_{#1}^{#2}\CondRectBrackets{#3}}}
\newcommand{\ExptFlatOrigHistory}[2]{\ExptFlat{t+1\CondColon{#1}}{p_Z}{#2}}
\newcommand{\ExptFlatOrigState}[2]{\ExptFlat{t+1\CondColon{#1}}{p_T,p_Z}{#2}}
\newcommand{\ExptFlatSimpHistory}[2]{\ExptFlat{t+1\CondColon{#1}}{q_Z}{#2}}
\newcommand{\ExptFlatSimpState}[2]{\ExptFlat{t+1\CondColon{#1}}{p_T,q_Z}{#2}}

\newcommand{\DZ}{\Delta_{Z}}  
\newcommand{\DZEst}{\Estim{\Delta}_{Z}}  
\newcommand{\MImSa}{\RewardEstim{m}}  
\newcommand{\PhiImSa}{\RewardEstim{\Phi}}
\newcommand{\PhiImSaEst}{\Estim{\Phi}}
\newcommand{\PhiImSaEstOrig}{\Estim{\PhiImSa}}

\newcommand{\OrigPOMDP}{\mathbf{P}}
\newcommand{\PBMDP}{\mathbf{M_{P}}}
\newcommand{\PB}[1]{\bar{b}_{#1}}
\newcommand{\Policy}[1]{\pi_{#1}}  
\newcommand{\PiLB}{\pi^{\mathcal{LB}}}
\newcommand{\PiUB}{\pi^{\mathcal{UB}}}

\theoremstyle{plain}
\newtheorem{theorem}{\protect\theoremname}
\theoremstyle{plain}
\newtheorem{lemma}{\protect\lemmaname}
\theoremstyle{plain}
\newtheorem{corollary}{\protect\corollaryname}
\providecommand{\theoremname}{Theorem}
\providecommand{\lemmaname}{Lemma}
\providecommand{\corollaryname}{Corollary}

\theoremstyle{plain}
\newtheorem{theoremrpt}{\protect\theoremrptname}
\theoremstyle{plain}
\newtheorem{lemmarpt}{\protect\lemmarptname}
\theoremstyle{plain}
\newtheorem{corollaryrpt}{\protect\corollaryrptname}
\providecommand{\theoremrptname}{Theorem}
\providecommand{\lemmarptname}{Lemma}
\providecommand{\corollaryrptname}{Corollary}

\newenvironment{innerproof}
 {\proof}
 {\endproof}

\title{Simplifying Complex Observation Models in Continuous POMDP Planning with Probabilistic Guarantees and Practice}
\author {
    Idan Lev-Yehudi\textsuperscript{\rm 1},
    Moran Barenboim\textsuperscript{\rm 1},
    Vadim Indelman\textsuperscript{\rm 2}
}
\affiliations {
    \textsuperscript{\rm 1}Technion Autonomous Systems Program (TASP), Technion - Israel Institute of Technology, Haifa 32000, Israel\\
    \textsuperscript{\rm 2}Department of Aerospace Engineering, Technion - Israel Institute of Technology, Haifa 32000, Israel\\
    \{idanlev, moranbar\}@campus.technion.ac.il,
    vadim.indelman@technion.ac.il
}

\begin{document}

\maketitle

\begin{abstract}
    Solving partially observable Markov decision processes (POMDPs) with high dimensional and continuous observations, such as camera images, is required for many real life robotics and planning problems. 
    Recent researches suggested machine learned probabilistic models as observation models, but their use is currently too computationally expensive for online deployment.
    We deal with the question of what would be the implication of using simplified observation models for planning, while retaining formal guarantees on the quality of the solution. 
    Our main contribution is a novel probabilistic bound based on a statistical total variation distance of the simplified model. 
    We show that it bounds the theoretical POMDP value w.r.t. original model, from the empirical planned value with the simplified model, by generalizing recent results of particle-belief MDP concentration bounds.
    Our calculations can be separated into offline and online parts, and we arrive at formal guarantees without having to access the costly model at all during planning, which is also a novel result. 
    Finally, we demonstrate in simulation how to integrate the bound into the routine of an existing continuous online POMDP solver.
\end{abstract}

\section{Introduction}

\begin{figure}[t]
    \centering
    \includegraphics[width=0.8\columnwidth]{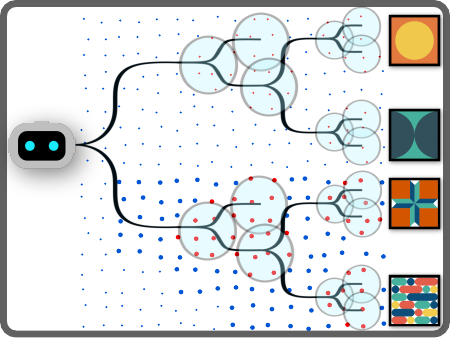} 
    \caption{
        An illustration of a planning session with a simplified observation model.
        The scattered dots are the pre-sampled states, and the dot size is relative to $\DZ$, the estimated discrepancy between the simplified and original observation models.
        The simplified observation model is less accurate on the bottom where the surroundings are more visually complex.
        For the two policies, we compute the bound as a summation over $\DZ$ weighted by the transition model. 
        We bound the summation to a truncation distance indicated by the cyan circles, and $\DZ$ within it is marked in red.
        The bottom policy chooses actions that give higher weights to states with greater $\DZ$, resulting in looser bounds.
    }
    \label{fig:planning_illustration}
\end{figure}

\begin{figure}[t]
    \centering
    \includegraphics[width=0.95\columnwidth]{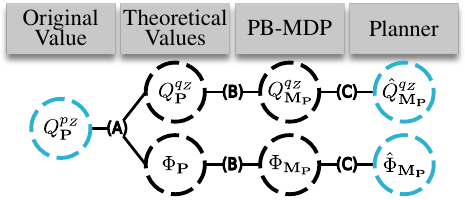}
    \caption
    {
    The various relationships between the action value functions in Corollary \ref{crl:JointApproximationBound} for probably-approximately bounding $\absvalflat{Q_{\OrigPOMDP}^{p_Z}-\Estim{Q}_{\PBMDP}^{q_Z}}\leq \Estim{\Phi}_{\PBMDP}$.
    (A) is given by Theorem \ref{thm:LocalActionBound}, connecting theoretical value functions with the theoretical local state bound.
    (B) is given by Corollary \ref{crl:ArbitraryPrecisionBounds}, connecting theoretical action value functions with their PB-MDP approximation.
    (C) is given by any planner with performance guarantees, such as POWSS, approximating the PB-MDP values.
    }
    \label{fig:inkscape_figure}
\end{figure}

Partially observable Markov decision processes (POMDP) are a flexible mathematical framework for modeling real world decision-making and planning problems with inherent uncertainty.
Yet, POMDPs are notoriously hard to solve \cite{Papadimitriou87math}.
Online solvers like POMCP \cite{Silver10nips} and DESPOT \cite{Ye17jair} focus on finding a solution for the current belief only, rather than solving for the entire belief space, hereby restricting the complexity of solution.

These planners are not trivially suitable for problems with continuous observation or action spaces.
In recent years there have been several practical approaches for continuous solvers, such as POMCPOW, PFT-DPW \cite{Sunberg18icaps} and LABECOP \cite{Hoerger21icra}.
However the current state of the art is divided between practical algorithms and those with guarantees \cite{Lim21cdc}.

Recently were suggested deep neural networks for probabilistic visual observation models in AI.
\cite{Eslami18neural} suggested a neural model for scene rendering.
Both \cite{Jonschkowski18rss} and \cite{Karkus18corl} suggested DPF and PF-NET respectively, which are models aimed at learning particle filtering.
In the context of model free planning, examples include QMDP-net \cite{Karkus17nips} and Deep Variational Reinforcement Learning \cite{Igl18pmlr}.
In the context of model based planning, recent approaches include DualSMC \cite{wang20ijcai} and VTS \cite{Deglurkar23l4dc}.
All of these works demonstrate the computational challenges arising from incorporating deep neural visual models into practical POMDP planning.

Our research focuses on the implications of using a different observation model, less accurate but computationally favorable, during planning in continuous POMDPs.
This could describe for example the case of using a shallower neural network for planning than for inference.
To the best of our knowledge, there still isn't a clear-cut answer as to how accurate a learned observation model needs to be in order to attain a target performance in continuous POMDP planning.
Our work provides insight into the complexity-performance trade-off in planning with different observation models.

\subsection{Contribution}

In this paper, we employ a simplified observation model and derive guarantees in the form of concentration inequalities.
To the best of our knowledge, this is the first to work to do so.
Specifically,
\begin{itemize}
    \item We derive a novel non-parametric bound on the difference between value functions when planning with different observation models, and we observe that it can be formulated as a local state function.
    \item Our bound is practically computed by separating calculations to offline and online parts, such that the online estimator of the bounds refrains from accessing the computationally expensive original observation model.
	\item We derive concentration bounds on the online estimator of the bound.
    We do so by generalizing previous convergence results of particle-belief MDPs to state rewards under general policies.
    \item Finally, we demonstrate the practical computation of the bound's estimator in a simple simulated environment.
\end{itemize}

\subsection{Related Work}

\paragraph{Simplification of Probabilistic Models in POMDP Planning}
In \cite{Hoerger23ijrr} the authors introduce considered several simplification levels of the transition model, whereas we consider only two levels.
Yet \citeauthor{Hoerger23ijrr} had to access the most complex model online, and their convergence guarantee is only asymptotic.
\cite{Ha18nips} consider learning a simplified generative model, for environments in which direct training would be infeasible.
On this "world model" they train the policy, yet they too do not have any performance guarantees for the learned policy.
A technique for measuring non-linearity based on total variation distance in POMDP planning was considered by \cite{hoerger20wafr}.
While there are some similarities in the approach to our work, they did not show how their practical estimator is related to the theoretical bound, nor did they give any performance guarantees.

\paragraph{Continuous POMDP Planning With Guarantees}
\cite{Lim20ijcai} proved that POWSS, a Monte Carlo tree search algorithm for continuous POMDPs based on a modification of Sparse Sampling \cite{Kearns02jml}, converges to the optimal policy in high probability. However, is not efficient enough for practical use.
Later in \cite{Lim23jair} the convergence results were extended to prove that generally the particle-belief MDP (PB-MDP) accurately approximates the original POMDP with high probability.
Recently \cite{Shienman22icra, Barenboim23ral2} developed pruning of data association hypotheses in continuous POMDPs with guarantees.
The former focuses on a specific reward of entropy of hypotheses weights, while the latter provides looser bounds but for any general state reward.
In both works, the hybrid belief setting and type of simplification are different from our work.

\section{Preliminaries}

A POMDP is the tuple $\langle \mathcal{X},\mathcal{A},\mathcal{Z},p_{T},p_{Z},r,\gamma,L,b_0\rangle$.
$\mathcal{X},\mathcal{A},\mathcal{Z}$ are the state space, action space and observation space respectively. Throughout the paper, $x\in\mathcal{X}$, $a\in\mathcal{A}$, $z\in\mathcal{Z}$ denote individual states, actions and observations respectively.
We consider $\mathcal{X}$ and $\mathcal{Z}$ that are continuous, while $\mathcal{A}$ can be either discrete or continuous

The conditional probability distributions $p_{T}$ and $p_{Z}$ are the transition and observation models respectively. $p_{T}(x^{\prime}\mid x,a)$ models the uncertainty of taking action $a\in\mathcal{A}$ from state $x\in\mathcal{X}$, and $p_{Z}(z\mid x)$ models the uncertainty of receiving an observation $z\in\mathcal{Z}$ at state $x$. 
The reward function $r\colon\mathbb{N}\times\mathcal{X}\times\mathcal{A}\rightarrow\mathbb{R} $ gives the immediate reward of applying action $a$ at state $x$ and time $t$. 
We denote $r_t\colon \mathcal{X}\times\mathcal{A}\to \mathbb{R}$ as the reward function at time $t$. 
We assume that the POMDP starts at time $0$ and terminates after $L\in\mathbb{N}$ steps. 
$\gamma\in(0,1]$ is the discount factor. 

Because of the partial observability, the agent has to maintain a probability distribution of the current state given past actions and observations, known as the belief.
The initial belief of the agent, denoted as $b_{0}$, captures the uncertainty in the initial state.
A history at time $t$ is defined as a sequence of the starting belief, followed by actions taken and observations received until $t$: $H_{t}\bydef(b_{0},a_{0},z_{1},\dots,a_{t-1},z_{t})$.
The belief at time $t$ is defined as the conditional distribution of the state given the history $b_{t}(x_{t})\bydef \probd (x_{t}\mid H_{t})$.
We denote $H_{t}^{-}\bydef(b_{0},a_{0},z_{1},\dots,a_{t-1})$ for the same history without the last measurement, and the propagated belief $b_{t}^{-}(x_{t})\bydef \probd (x_{t}\mid H_{t}^{-})$.
The belief reward is the expected state reward: $r_i(b_i, a)\bydef \ExptFlat{x_i\sim  b_i}{}{r_i(x_i,a)}$.

A policy at time $i$, denoted as $\Policy{i}$, is a mapping from the space of all histories to the action space $\Policy{i}\colon \mathcal{H}_i\to \mathcal{A}$. A policy $\pi$ is a collection of policies from the start time until the POMDP terminates, i.e. $\pi\bydef(\pi_t)_{t=0}^{t=L-1}$, and for brevity we denote $\Policy{i}$ instead of $\Policy{i}(H_i)$.
It can be shown that optimal decision-making can be made given the belief, instead of considering the entire history \cite{Kaelbling98ai}.
Since the belief and history fulfill the Markov property, a POMDP is a Markov decision process (MDP) on the belief space, commonly referred to as the belief MDP.

The value function of a policy from time $t$ is the expected sum of discounted rewards until the horizon of the POMDP: $V_{t}^{\pi}(b_{t})\bydef \ExptFlat{z_{t+1:L}}{}{\sum_{i=t}^{L}\gamma^{i-t}r_{i}(b_{i},\Policy{i})}$.
Most often the action value function, defined as $Q_{t}^{\pi}(b_{t},a)\bydef r(b_{t},a)+\gamma\mathbb{E}_{z_{t+1}}[V_{t+1}^{\pi}(b_{t+1})]$, is used as an intermediate step in the estimation of the value function.
The goal of planning is to compute a policy that would maximize the value function.
The optimal policy is denoted $\pi^{*}=\argmax_{\pi} V_{t}^{\pi}(b_{t})$, and the corresponding value and action value functions for the optimal policy are $V_{t}^{*}$ and $Q_{t}^{*}$.

Often in real world applications it is impractical to update the belief exactly.
Particle filters are used instead to represent a belief non parametrically.
The particle belief of $C$ particles at time $i$ is represented by a set state samples and weights $\PB{t} =\{(x_{t}^{j},w_{t}^{j})\}_{j=1}^{C}$, and is defined as the discrete distribution $\probd (x\mid \PB{t})=\frac{\sum_{j=1}^{C}{w_{t}^{j}\cdotp \delta(x-x_{t}^{j})}}{\sum_{j=1}^{C}{w_{t}^{j}}}$.
The basic particle filtering algorithm, Sequential Importance Sampling (SIS) \cite[1.3.2]{Doucet00}, approximates the target posterior distribution by recursively updating the importance weights of the particles with the observation likelihood, and normalizes the weights accordingly.
Sequential Importance Resampling (SIR) adds resampling of state particles after the weight updates, to avoid weight degeneracy \cite{Kong94}.

The Monte Carlo (MC) estimator of an expectation, denoted with $\Estim{\ExptFlat{}{}{}}$, is its approximation with a finite number of samples. This applies to quantities that are defined as expectations, such as $\Estim{V}_{t}^{\pi}$.
For example, the empirical value function based on $N$ samples $\{z_{t+1:L}^{j}\}_{j=1}^{N}$ would be $\Estim{V}_{t}^{\pi} = \frac{1}{N} \sum_{j=1}^{N} \gamma^{i-t}r_{i}(b_{i}^{j},\Policy{i}^{j})$, $\Policy{i}^{j}=\pi(H_{i}^{j})$, where $b_{i}^{j}$ and $H_{i}^{j}$ are updated with $z_{t+1:i}^{j}$, i.e. the sampled sequence of observations until time $i$.

\cite{Lim23jair} considered the PB-MDP, where the belief space is explicitly the space of particle beliefs $\PB{}$ of $C$ particles.
They proved that with high probability the optimal value in the PB-MDP is bounded from the optimal POMDP value, and this bound converges as $C\to\infty$.

We similarly denote the original POMDP as $\OrigPOMDP$ and the PB-MDP as $\PBMDP$.
Quantities that are measured w.r.t. each are denoted accordingly, f.e. the value function starting at time $t$, given policy $\pi$ and starting belief $b_{t}$ in the original POMDP is denoted as $V_{\OrigPOMDP,t}^{\pi}(b_{t})$.
When a result applies to both $\OrigPOMDP$ and $\PBMDP$, we omit this notation.

\section{Methodology}

\subsection{Problem Formulation} 
\label{sec:ProblemFormulation}

We consider cases in which the observation model $p_{Z}$ is simplified for computational reasons.
We denote the simplified model as $q_{Z}$, which is only used in planning.
In our setting, inference is performed with the original model, such that the initial belief $b_{t}$ for planning session starting at time $t$ was updated with $p_{Z}$.
We denote the future value function based on belief updates in planning with either the original or simplified observations models as $V_{t}^{\Policy{},p_Z}$ and $V_{t}^{\Policy{},q_Z}$ respectively, and similarly for $Q_{t}^{\Policy{},p_Z}$ and $Q_{t}^{\Policy{},q_Z}$.

In the settings we consider, $p_Z$ is computationally more expensive than $q_Z$, and both are considerably more expensive than $p_T$.
In the example of a ground robot equipped with a camera, $q_Z$ could be a neural network of similar architecture to $p_Z$ but much shallower, whereas $p_T$ might be a Gaussian or some other simple parametric distribution. Note that in part the difference in complexity arises from the difference in dimensions of the state and observation spaces.

Our goal is to derive calculable probabilistic bounds on $\absvalflat{V_{t}^{\Policy{},p_Z} - \Estim{V}_{t}^{\Policy{},q_Z}}$, while restricting access to $p_Z$ to offline computations only.
Additionally, we wish to analyze the difference in using the bounds for post guarantees, i.e. after a policy has been extracted, to when they're used during the decision-making, i.e. policy computation.

We denote shortened notations for the following densities:
 $\tau_{i} \bydef p_T(x_{i}\mid x_{i-1},\Policy{i-1})$, $[\zeta\slash\xi]_{i} \bydef [{p_Z}\slash{q_Z}](z_i \mid x_{i})$, $[\zeta\slash\xi]_{i}^{H} \bydef [{p_Z}\slash{q_Z}](z_i \mid H_{i}^{-})$. 
 The values in $[\cdot]$ can be replaced with either respective option. 
 The product of densities is denoted $[\tau\slash\zeta\slash\xi]_{i:j} \bydef \prod_{l=i}^{j}{[\tau\slash\zeta\slash\xi]_{l}}$. 
 We denote the expectations:
 \begin{gather}
    \textstyle \ExptFlat{b_{i}}{}{\star} \bydef \int_{x_{i}} b_i(x_i) \star \dif x_i \\
    \textstyle \ExptFlat{i:j}{p_T}{\star} \bydef \int_{x_{i:j}} \tau_{i:j}\star \dif x_{i:j} \\
    \textstyle \ExptFlat{i:j}{[{p_Z}\slash{q_Z}]}{\star} \bydef \int_{z_{i:j}} [\zeta\slash\xi]_{i:j}^{H} \star \dif z_{i:j} \\
    \textstyle \ExptFlat{i:j}{p_T,[{p_Z}\slash{q_Z}]}{\star} \bydef \int_{z_{i:j}} \ExptFlat{b_{i-1}}{}{\ExptFlat{i:j}{p_T}{[\zeta\slash\xi]_{i:j} \star}} \dif z_{i:j} \\
    \textstyle \ExptFlat{i:{j+1}^{-}}{p_T,[{p_Z}\slash{q_Z}]}{\star} \bydef \int_{z_{i:j}} \ExptFlat{b_{i-1}}{}{\ExptFlat{i:j+1}{p_T}{[\zeta\slash\xi]_{i:j} \star}} \dif z_{i:j},
\end{gather}
It holds that $\textstyle [\zeta\slash\xi]_{i:j}^{H}=\ExptFlat{b_{i-1}}{}{\ExptFlat{i:j}{p_T}{[\zeta\slash\xi]_{i:j}}}$ by the law of total probability,
and it follows that $\ExptFlat{i:j}{[{p_Z}\slash{q_Z}]}{\star}=\ExptFlat{i:j}{p_T,[{p_Z}\slash{q_Z}]}{\star}$.
Additionally, it holds that $\ExptFlat{i:j}{p_T,[{p_Z}\slash{q_Z}]}{\ExptFlat{j+1}{p_T}{\star}} = \ExptFlat{i:{j+1}^{-}}{p_T,[{p_Z}\slash{q_Z}]}{\star}$ by Fubini's theorem \cite[1.7]{Durrett19book}.

We denote for a bounded quantity its lower and upper bounds as $\mathcal{LB}$ and $\mathcal{UB}$ respectively. F.e, if $\absvalflat{A}\leq B$, then $\mathcal{LB}(A)=-B$ and $\mathcal{UB}(A)=B$.

For the formulation and practical computation of the bounds, we take the following assumptions:
\begin{enumerate}[label=\roman*.]
    \item The reward at each time step is bounded: $\absvalflat{r_i(x_{i},a)}\leq \Rmax{i}$, for time index $i$. Hence, follows from the triangle inequality that the value function at a certain time step is bounded by $\absvalflat{V_{t}^{\pi}(b_{t},a)}\leq \sum_{i=t}^{L}\gamma^{i-t}\Rmax{i} \bydef \Vmax{t}$, and we define $V_{\max}\bydef \max_t \Vmax{t}$. \label{assum:BoundedReward}
    \item The reachable state space is totally bounded, i.e. for every $\varepsilon > 0$ it can be covered with a finite number of open balls of radius $\varepsilon$. \label{assum:TotallyBounded}
    \item The models $p_{T}$, $p_{Z}$, $q_{Z}$ can be sampled from, and queried for PDF values given samples. \label{assum:ModelsSampleQuery}
\end{enumerate}
For the result of probabilistic convergence guarantees in Theorem \ref{thm:ProbConverge}, we additionally assume the following:
\begin{enumerate}[label=\roman*., resume]
    \item Weights of particles in particle beliefs are updated via the Sequential Importance Sampling (SIS) algorithm, i.e. $w_{t}^{j}\propto [\zeta\slash\xi]_{t}\cdotp w_{t-1}^{j}$ for the $j$'th observation sample $z_{t}^{j}$. Specifically, there is no resampling step. \label{assum:PFSIS}
    \item The normalized observation likelihoods of the original and simplified models are bounded almost everywhere w.r.t. the particle filter's proposal distribution: $\sup_{z_{1:t}} \esssup_{x_{0:t}\sim b_0\cdotp \tau_{1:t}}\frac{[\zeta\slash\xi]_{1:t}}{[\zeta\slash\xi]_{1:t}^{H}}\leq d_{\infty}^{\max}$, where $d_{\infty}^{\max}\in\mathbb{R}$. See \cite{Lim23jair} for further details. \label{assum:dMax}
\end{enumerate}

\subsection{Bound for Simplified Observation Model}

In this section we derive the theoretical bound for the value function with the original observation model and a given policy, w.r.t. the value with the simplified model.
We start by stating the following known equivalence relationship between expectation over rewards.
\begin{lemma}
\label{lem:StateExptEquiv}
The expected belief-dependent reward w.r.t. histories, is equivalent to the expected state-dependent reward w.r.t. the joint distribution of states and observations.
\begin{gather}
    \ExptFlat{t+1:i}{[{p_Z}\slash{q_Z}]}{r_i(b_{i},\Policy{i})}=\ExptFlat{t+1:i}{p_T,[{p_Z}\slash{q_Z}]}{r_i(x_{i},\Policy{i})}. \label{eq:state_action_reward}
\end{gather}
\end{lemma}
\begin{innerproof}
    We refer to the supplementary material for all proofs\ifthenelse{\value{arxiv}>0}{}{ \cite{LevYehudi2023arxiv}}.
\end{innerproof}
We define the following state dependent total variation distance (TV-distance) function,
\begin{gather}
    \textstyle \DZ(x)\bydef 
    \int_{\mathcal{Z}}\absvalflat{p_{Z}(z\mid x)-q_{Z}(z\mid x)}\dif z. \label{eq:def_delta}
\end{gather}
There are several motivations for considering $\DZ(x)$.
First, we can estimate it via samples for any density from which we can sample and evaluate for its PDF.
It does not require any parametric form of the density, which is useful when considering general learned models.
Second, it is state-dependent, hence can be computed locally as we show later for a given belief or trajectory in the belief tree.
Different actions might result in different bounds and the locality helps differentiate actions that result in tighter or looser bounds.
Third, as given by Pinsker's lemma, the TV-distance is bounded from above by the KL-divergence \cite{tsybakov2009book}.
Thus, current approaches that learn probabilistic models with ELBO or directly minimize empirical KL-divergence \cite{Kingma14iclr, Sohn2015nips, Rezende15icml, Winkler2019arxiv} also indirectly minimize the TV-distance, therefore it is appropriate to assume that models trained with these objectives will also indirectly minimize $\DZ$.

\begin{theorem}
\label{thm:AltModelThBound}
Assume current belief is $b_{t}$ and a given policy is $\pi$.
Denote with $b_{i}^{p_Z}$, $b_{i}^{q_Z}$ the future belief at time step $i$ updated with either $p_Z$ or $q_Z$, respectively. Then it holds that
\begin{gather}
    \textstyle \absvalflat{\ExptFlatOrigHistory{i}{r_{i}(b_{i}^{p_Z},\Policy{i})}-\ExptFlatSimpHistory{i}{r_{i}(b_{i}^{q_Z},\Policy{i})}} \nonumber \\
    \textstyle \leq \Rmax{i}\sum_{l=t+1}^{i}\ExptFlatSimpState{l^-}{\DZ(x_{l})},
    \label{eq:state_bound_reward}
\end{gather}
i.e. the difference between the expected reward at future time $i$ for the original and simplified POMDP is bounded by the maximum reward, times the sum of the expected state-dependent TV-distances between the observation models.
\end{theorem}
By applying to the value function with the triangle inequality, and changing the order of summation, we arrive at the following result.
\begin{corollary} 
    \label{crl:TheoreticalBound}
    The difference between the original and simplified value functions can be bounded by the following sum of scaled expected TV-distance terms:
    \begin{gather}
        \textstyle \absvalflat{V_{t}^{\pi,p_Z}(b_{t})-V_{t}^{\pi,q_Z}(b_{t})}\nonumber\\
        \textstyle \leq\sum_{i=t+1}^{L}\ExptFlatSimpState{i^-}{\DZ(x_{i})}\sum_{l=i}^{L}\gamma^{l-t}\Rmax{l} \\
        \textstyle =\sum_{i=t+1}^{L}\Vmax{i}\cdotp\ExptFlatSimpState{i^-}{\DZ(x_{i})}. 
        \label{eq:bound_value}
    \end{gather}
\end{corollary}

\subsection{Equivalence to State-Action Local Bound}
Corollary \ref{crl:TheoreticalBound} requires computing the theoretical expectations $\ExptFlatSimpState{i^-}{\DZ(x_{i})}$, which is not trivial with continuous states and observations.
Our key insight is that the bound can be viewed as the expected sum of a state-action function.
This serves two purposes.
The first is that the bounds become more easily estimated, as we can integrate our calculations in existing POMDP planners by adding a secondary reward-like function to compute.
The second is that we can extend convergence results of POMDP algorithms to the empirical estimate of the bound.

First we define the following time dependent state-action function, and its extension to a belief function,
\begin{gather}
    m_{i}(x_i,\Policy{i})\bydef\Vmax{i+1}\cdotp\ExptFlat{i+1}{p_T}{\DZ(x_{i+1})},\\
    m_{i}(b_{i},\Policy{i})\bydef\ExptFlat{b_i}{}{m(x_i,\Policy{i})}.
\end{gather}
Intuitively speaking, $m_i$ bounds the loss of the value function at time $i$ when considering the action of $\Policy{i}$ from state $x_i$ or belief $b_i$, based on $\DZ$.
It is then natural to extend this definition with the following:
\begin{enumerate}
    \item Cumulative bound for a given policy and initial belief, $M_{t}^{\pi}(b_{t})\bydef\ExptFlatSimpHistory{L-1}{\sum_{i=t}^{L-1}m_{i}(b_{i},\Policy{i})}$, analogous to $V_{t}^{\pi}(b_{t})$.
    Note that the time horizon has decreased by $1$, and there is no discount factor.
    \item Action cumulative bound for a given policy and initial belief, $\Phi_{t}^{\pi}(b_{t},a)\bydef m_{t}(b_{t},a)+\ExptFlatSimpHistory{}{M_{t+1}^{\pi}(b_{t+1})}$, analogous to $Q_{t}^{\pi}(b_{t},a)$.
\end{enumerate}
\begin{theorem}
    \label{thm:LocalActionBound}
    Under the conditions of Theorem \ref{thm:AltModelThBound}, the difference between the original and simplified value function can be bounded by
    \begin{gather}
        \textstyle \absvalflat{V_{t}^{\pi,p_Z}(b_{t})-V_{t}^{\pi,q_Z}(b_{t})}\leq M_{t}^{\pi}(b_{t}). 
        \label{eq:MBoundV}
    \end{gather}
    In addition, the respective difference in action value function can be bounded by the action cumulative bound,
    \begin{gather}
        \textstyle \absvalflat{Q_{t}^{\pi,p_Z}(b_{t},a)-Q_{t}^{\pi,q_Z}(b_{t},a)}\leq \Phi_{t}^{\pi}(b_{t},a). 
        \label{eq:PhiBoundQ}
    \end{gather}
\end{theorem}

Computing $M_{t}^{\pi}(b_{t})$ would be useful when trying to arrive at a bound for a specific policy, usually as post-guarantees after a planning session.
On the other hand, the computation of $\Phi_{t}^{\pi}(b_{t},a)$ can be defined recursively over an entire belief tree, meaning it extends the original definition of the bound to all extractable policies from a given tree.
Hence, the action cumulative bound can be used during planning, when computing a policy or pruning low value branches like in \cite{Sztyglic22iros}.

\subsection{Online Estimator of Local Bound}

We now show how we estimate $m_i$ during online planning without requiring access to $p_Z$.
We do this by offline pre-sampling state samples $\{x_{n}^{\Delta}\} _{n=1}^{N_{\Delta}}\sim Q_{0}(x)$, named delta states, and evaluating $\DZ(x_{n}^{\Delta})$ for each.
During online planning, we simply reweight $\DZ$ using the importance sampling formalism for state sample $x_{i}^{j}$.
We use this to redefine $m_i$ and also to define the empirical estimate $\RewardEstim{m}_{i}$.
\begin{gather}
    \textstyle m_{i}(x_{i}^{j},a)=\Vmax{i+1}\cdotp\ExptFlat{x^{\prime}\sim Q_{0}}{}{\frac{p_{T}(x^{\prime}\mid x_{i}^{j},a)}{Q_{0}(x^{\prime})}\DZ(x^{\prime})} \label{eq:StateBoundImpSamp},\\
    \textstyle \MImSa_{i}(x_{i}^{j},a)\bydef\Vmax{i+1}\cdotp \frac{1}{N_{\Delta}}\sum_{i=1}^{N_{\Delta}}
    \frac{p_{T}(x_{n}^{\Delta}\mid x_{i}^{j},a)}{Q_{0}(x_{n}^{\Delta})}\DZ(x_{n}^{\Delta}). \label{eq:StateBoundApprox}
\end{gather}
Based on this, we define $\PhiImSa_{t}^{\pi}(b_{t},a)$ when computed with $\RewardEstim{m}_i$ instead of $m_i$. We explicitly define this notation to differentiate $\PhiImSa$ from $\PhiImSaEstOrig$;
The former is defined with a theoretical expectation over the observation space, while the latter is an empirical estimate with sampled belief trajectories. For notational brevity, in the rest of the paper we denote $\PhiImSaEst\bydef \PhiImSaEstOrig$.

The support of $Q_{0}$ must cover all of $\mathcal{X}$ for the estimator in \eqref{eq:StateBoundApprox} to be consistent \cite[1.3.2]{Doucet00}.
Although it is possible to construct such densities over unbounded state spaces, we find that in practice it is often not required as $\mathcal{X}$ is either naturally bounded, such as the configuration space of a manipulator, or can be defined to be large enough but bounded to cover the relevant domain of the planning problem.
Hence, it could be sufficient to choose $Q_{0}$ with a finite support.
To the end of simplifying the discussion, we took assumption \ref{assum:TotallyBounded}

Under this condition, we can establish with Hoeffding's inequality the following simple probabilistic bound:
\begin{gather}
    \textstyle \OpProbFlat{\absvalflat{m_{i}(x,a)-\RewardEstim{m}_{i}(x,a)}\geq\nu}\leq2\exp(-2N_{\Delta}\frac{\nu^{2}}{B_{i}^{2}}),
\end{gather}
where $B_{i}=2\cdotp\Vmax{i}\cdotp\max_{x,x^{\prime},a}\frac{p_{T}(x^{\prime}\mid x,a)}{Q_{0}(x^{\prime})}$ since $0\leq\DZ(x)\leq2$. We assume $B_{i}$ to be computable offline.
This is sufficient for obtaining the concentration bounds of Theorem \ref{thm:ProbConverge}, we approximate this even further $N_{\Delta}$ may be very large, as we describe in section \ref{sec:comp_dz}.

\section{Convergence Guarantees}

Our major convergence result is based on an adaptation of Lemma 1 of \cite{Lim23jair}. 
Lemma 1 shows that the theoretical optimal POMDP action value $Q_{\OrigPOMDP}^{*}$ is bounded with high probability from the theoretical optimal PB-MDP action value $Q_{\PBMDP}^{*}$.
We extend their result in two ways:
First, proving the bound for general policies, rather than only the optimal one, allows for proving convergence of $\Phi_{t}^{\pi}$ w.r.t. a policy that is chosen to maximize $V_{t}^{\pi,q_Z}$, or any other objective.
Second, by introducing an additional assumption of a probabilistically bounded reward, denoted with $\RewardEstim{r}$, we are able to apply Theorem \ref{thm:ProbConverge} to $\RewardEstim{m}_i$.

\begin{theorem}[Generalized PB-MDP Convergence]
\label{thm:ProbConverge}
    Assume that the immediate state reward estimate is probabilistically bounded such that
    $\OpProbFlat{\absvalflat{r_{i}^{j}-\tilde{r}_{i}^{j}}\geq\nu}\leq\delta_{r}(\nu,N_r)$, for a number of reward samples $N_r$ and state sample $x_{i}^{j}$. 
    Assume that $\delta_{r}(\nu,N_r)\to0$ as $N_r\to\infty$.
    For all policies $\pi$, $t=0,\dots,L$ and $a\in\mathcal{A}$, the following bounds hold with probability of at least $1-5(4C)^{L+1}(\exp(-C\cdot\acute{k}^{2})+\delta_{r}(\nu,N_r))$:
    \begin{gather}
        \textstyle
        \absvalflat{Q_{\OrigPOMDP,t}^{\pi,[{p_Z}\slash{q_Z}]}(b_{t},a)-Q_{\PBMDP,t}^{\pi,[{p_Z}\slash{q_Z}]}(\PB{t}, a)}\leq \alpha_{t}+\beta_{t},
    \end{gather}
    where,
    \begin{gather}
        \textstyle
        \alpha_{t}=(1+\gamma)\lambda+\gamma \alpha_{t+1},\ \alpha_{L}=\lambda\geq0, \\
        \beta_{t}=2\nu+\gamma \beta_{t+1},\ \beta_{L}=2\nu\geq 0, \\
        k_{\max}(\lambda, C)=\frac{\lambda}{4V_{\max}d_{\infty}^{\max}}-\frac{1}{\sqrt{C}}>0,\\
        \acute{k}=\min\{k_{\max},\lambda\slash4\sqrt{2}V_{\max}\}.
    \end{gather}
    If we require the bound to hold for all possible policies that can be extracted from a given belief tree simultaneously, then under the assumption of a finite action space, the probability is at least $1-5(4\absvalflat{\mathcal{A}}C)^{L+1}(\exp(-C\cdot\acute{k}^{2})+\delta_{r}(\nu,N_r))$.
\end{theorem}

\begin{corollary}
    \label{crl:ArbitraryPrecisionBounds}
    For arbitrary precision $\varepsilon$ and accuracy $\delta$ we can choose constants $\lambda, \nu, C, N_r$ such that the following holds with probability of at least $1-\delta$:
    \begin{gather}
        \absvalflat{Q_{\OrigPOMDP,t}^{\Policy{},[{p_Z}\slash{q_Z}]}(\PB{t},a)-Q_{\PBMDP,t}^{\Policy{},[{p_Z}\slash{q_Z}]}(\PB{t}, a)}\leq \varepsilon.
    \end{gather}
\end{corollary}

By applying Corollary \ref{crl:ArbitraryPrecisionBounds} twice, once to the original reward function then once to the action cumulative bound, we can conclude that the estimated simplified PB-MDP value is probabilistically bounded from the original theoretical POMDP value.

We now arrive at our key theoretical result.
\begin{corollary} 
    \label{crl:JointApproximationBound}
    Assuming that $\mathcal{P}$ is an MDP planner that can approximate $Q$-values with arbitrary precision $\varepsilon^{\mathcal{P}}$ at an accuracy $1-\delta^{\mathcal{P}}$, we denote the precision and accuracy of the action value and action cumulative bound functions:
    \begin{gather}
        \OpProbFlat{\absvalflat{Q_{\PBMDP,t}^{\pi,q_Z}(\bar{b}_{t},a)-\Estim{Q}_{\PBMDP,t}^{\pi,q_Z}(\bar{b}_{t},a)}\leq\varepsilon^{\mathcal{P}}_{Q}}\geq 1-\delta^{\mathcal{P}}_{Q} \label{eq:PlannerBoundsQ}\\
        \OpProbFlat{\absvalflat{\Phi_{\PBMDP,t}^{\pi}(\bar{b}_{t},a)-\PhiImSaEst_{\PBMDP,t}^{\pi}(\bar{b}_{t},a)}\leq\varepsilon^{\mathcal{P}}_{\Phi}}\geq 1-\delta^{\mathcal{P}}_{\Phi} \label{eq:PlannerBoundsPhi}
    \end{gather}
    From Corollary \ref{crl:ArbitraryPrecisionBounds} it holds that we can choose constants $\lambda, \nu, C, N_r$ such that the following holds,
    \begin{gather}
        \OpProbFlat{\absvalflat{Q_{\OrigPOMDP,t}^{\pi,q_Z}(b_{t},a)-Q_{\PBMDP,t}^{\pi,q_Z}(\PB{t}, a)}\leq \varepsilon_{Q}}\geq 1-\delta_{Q} \label{eq:PBMDPBoundsQ}, \\
        \OpProbFlat{\absvalflat{\Phi_{\OrigPOMDP,t}^{\pi}(b_{t},a)-\RewardEstim{\Phi}_{\PBMDP,d}^{\pi}(\PB{t}, a)}\leq \varepsilon_{\Phi}}\geq 1-\delta_{\Phi}. \label{eq:PBMDPBoundsPhi}
    \end{gather}
    Then with probability of at least $1-(\delta_{Q}+\delta^{\mathcal{P}}_{Q}+\delta_{\Phi}+\delta^{\mathcal{P}}_{\Phi})$
    \begin{gather}
        \absvalflat{Q_{\OrigPOMDP,t}^{\pi,p_Z}(b_{t},a)-\Estim{Q}_{\PBMDP,t}^{\pi,q_Z}(\bar{b}_{t},a)}\leq \\
        \PhiImSaEst_{\PBMDP,t}^{\pi}(\bar{b}_{t},a) + \varepsilon_{Q} + \varepsilon^{\mathcal{P}}_{Q} + \varepsilon_{\Phi} + \varepsilon^{\mathcal{P}}_{\Phi}.
    \end{gather}
\end{corollary}

In summary, any planner that can approximate the PB-MDP values $Q_{\PBMDP,t}^{\pi,q_Z}(\PB{t},a)$ with high probability can also approximate the cumulative bound $\Phi_{\PBMDP, t}^{\Policy{}}(\PB{t}, a)$ with high probability, since it is mathematically formulated like a state reward.
Therefore, we can construct a bound for the theoretical value $Q_{\OrigPOMDP,t}^{\Policy{},p_Z}(b_{t},a)$ with high probability from online calculations that do not involve access to $p_Z$.

An important observation is that the probabilistic bound obtained by Theorem \ref{thm:ProbConverge} is independent of the chosen policy, i.e. constant w.r.t. the actions. 
Therefore, when using the bound in Corollary \ref{crl:JointApproximationBound} for decision-making, if our goal is to distinguish between actions that result in maximal lower or upper bound, i.e. $\max_a [\mathcal{LB}\slash\mathcal{UB}](Q_{\OrigPOMDP,t}^{\pi,p_Z}(b_{t},a)) $, we will obtain the same action choice by computation of $\max_a {\Estim{Q}_{\PBMDP,t}^{\pi,q_Z}(\bar{b}_{t},a) [- \slash +] \PhiImSaEst_{\PBMDP,t}^{\pi}(\bar{b}_{t},a)}$.

\section{Implementation}

\begin{figure*}[t]
    \centering
    \begin{subfigure}[b]{.495\textwidth}
      \centering
      \includegraphics[width=\linewidth]{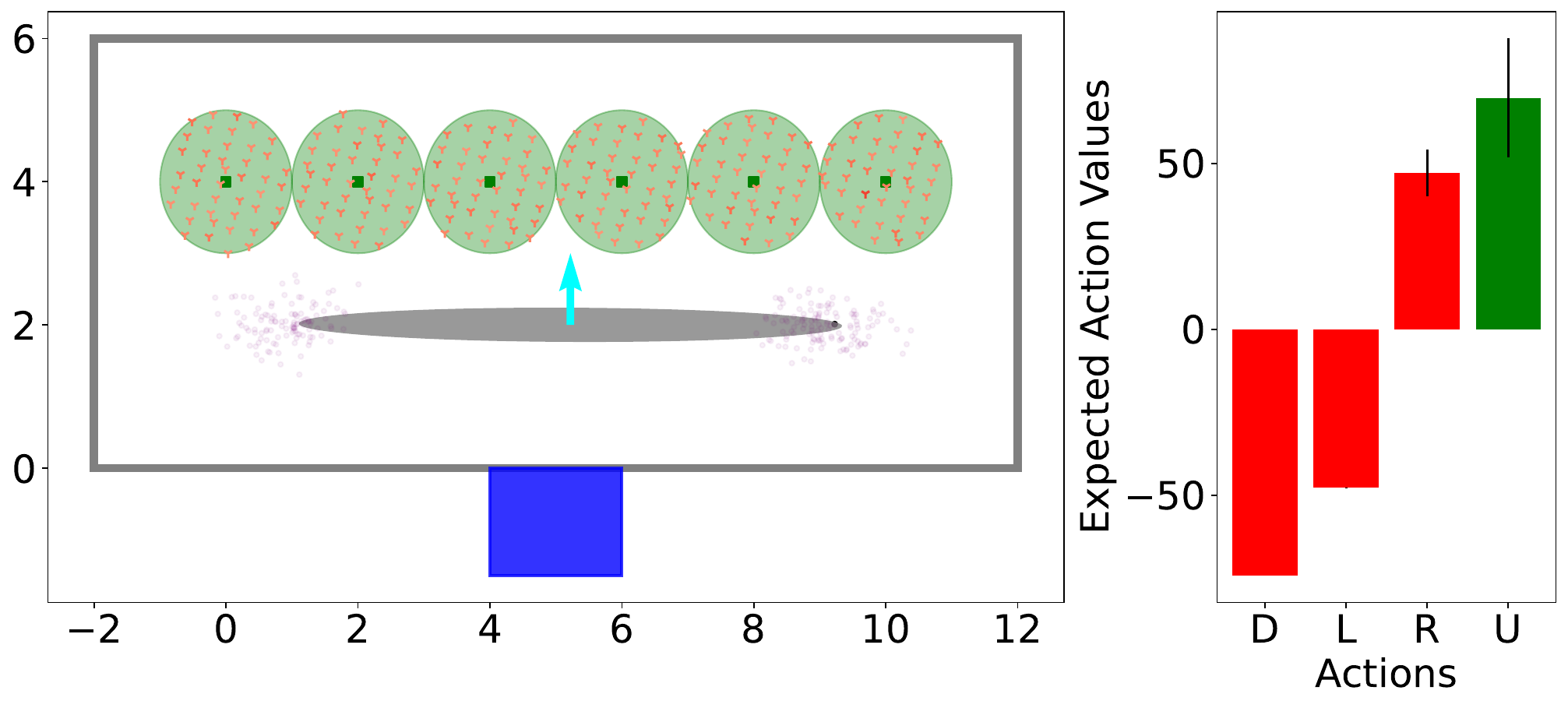}
      \caption{Start of scenario, $t=0$.}
      \label{fig:sub1}
    \end{subfigure}
    \begin{subfigure}[b]{.495\textwidth}
      \centering
      \includegraphics[width=\linewidth]{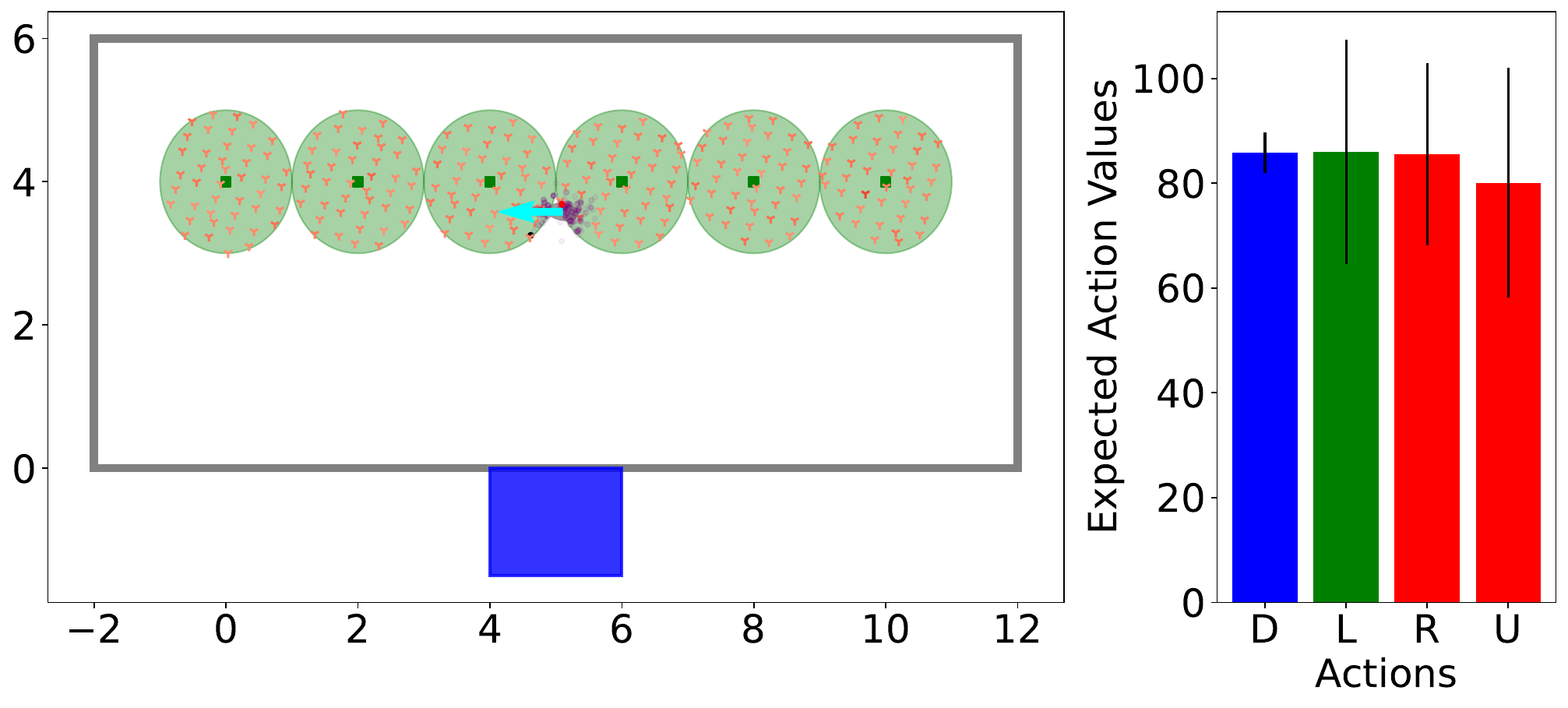}
      \caption{Middle of scenario, $t=4$.}
      \label{fig:sub2}
    \end{subfigure}
    \caption{
        The results of two planning sessions in 2D beacons.
        The goal is indicated by the blue rectangle, the beacons and their radii by the green squares and circles, and the outer walls by the grey outer rectangle.
        The filtered delta states $\{x_n^{\Delta}\}_{n=1}^{N_{\Delta}^{\textit{kept}}}$ are indicated by the tri-downs, with color relative to estimated TV-distance of the simplified observation model $7.06\leq\DZEst \cdotp10^{2}\leq 12.08$.
        The colored dots indicate: the true state in black, the observation in red and the particle belief in purple.
        The belief empirical mean and covariance are the grey ellipse.
        The bars on the right depict $\Estim{Q}_t^{q_Z}$ for all actions, and $\PhiImSaEst_t$ as symmetric error bars.
        The action chosen by the simplified value policy $\pi^{q_Z}_t$ is colored in green, and the lower bound policy $\PiLB_t$ in blue if different.
        At $t=4$ we can see an inconsistency of action order between $\pi^{q_Z}_t$ that chooses left, whereas $\PiLB_t$ chooses down.
    }
    \label{fig:environment}
\end{figure*}

We verify the computability of our approach in a 2D beacons POMDP. \ifthenelse{\value{arxiv}>0} {\footnote[1]{Our code is publicly available at \url{https://github.com/IdanLevYehudi/SimplifyingObsPOMDP}}} {}
We integrate the computation of $\MImSa_{i}$ and $\PhiImSaEst_{\PBMDP,t}$ into PFT-DPW, and showcase an example of where the bounds could affect a policy's decision-making. 

\subsection{Simulative Setting}

Our experimental setting is a 2D light-dark inspired simulation, shown in figure \ref{fig:environment}.

An agent is in a wall-surrounded arena, with a gate at the bottom to a goal region. The agent's starting location is either to the left or to the right of the goal region.
The agent's task is to enter the goal region without colliding with a wall.
The observations are noisy measurements of the agent's location, being more certain when in the "light" region, and less when in the "dark" region.

The light region $\mathcal{X}_{\text{light}}$ is defined by circles centered at each of the $6$ beacons located at the top of the arena, and $\mathcal{X}_{\text{dark}}=\mathcal{X}\setminus \mathcal{X}_{\text{light}}$.
The observation model in the light region $p_{Z}(z\mid x\in\mathcal{X}_{\text{light}})$ is defined as a Gaussian mixture model (GMM) with 1126 components arranged such that they approximately form a truncated Gaussian distribution centered at $x$.
The simplified observation model differs only in the light region, and it approximates $p_Z$ with a single Gaussian: $q_{Z}(z\mid x\in \mathcal{X}_{\text{light}})\sim\mathcal{N}(x,\Sigma_{\text{light}}^{q_Z})$.
We refer readers to \cite[1.4.16]{BarShalom04book} on approximating GMMs with a single Gaussian.
In the dark region, $p_Z(z\mid x\in \mathcal{X}_{\text{dark}}) = q_Z(z\mid x \in \mathcal{X}_{\text{dark}}) = \mathcal{N}(x, \Sigma_\textit{dark})$

This setting demonstrates a case where the original observation model is an overly parameterized model, like a complex neural network, and one would like to replace it with a less parameterized albeit similar model.

The action space is $\mathcal{A}=\{(\pm1,0),(0,\pm1)\}$. The transition model $p_T(x^\prime \mid x,a)$ is a Gaussian centered at the agent's location plus the action.
The horizon is $L=15$, and the POMDP will terminate early if the agent enters the goal region or collides with a wall.

The reward is only state and time dependent, and is a sum of three indicators: $r_{t}(x)=R_{\textit{hit}}\cdot\boldsymbol{1}_{x\in\mathcal{X}_{\textit{goal}}}+R_{\textit{miss}}\cdot\boldsymbol{1}_{x\notin\mathcal{X}_{\textit{goal}}}+R_{\textit{collide}}\cdot\boldsymbol{1}_{x\in\mathcal{X}_{\textit{collision}}}$.
In all time steps, $R_{\textit{hit}}=100$, $R_{\textit{collide}}=-50$.
The miss reward is $R_{\textit{miss}}=-50$ if $t=L$ and is $-1$ otherwise. The discount factor is $\gamma=1$.

Further details of the experimental setup can be found in the supplementary material.

\subsection{Implementation of Bounds} \label{sec:comp_dz}

\begin{figure}[t]
    \centering
    \includegraphics[width=0.8\columnwidth]{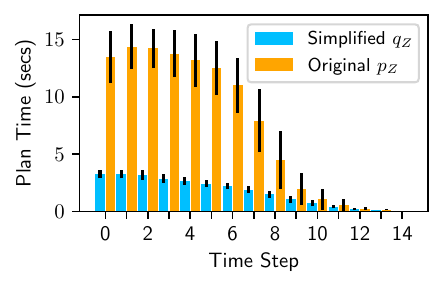}
    \caption{
        Mean and standard deviation of planning duration over 100 scenarios vs. scenario time step, with the original observation model $p_Z$ or simplified model $q_Z$.
        }
    \label{fig:results_times}
\end{figure}

When sampling delta states $\{x_{n}^{\Delta}\} _{n=1}^{N_{\Delta}}$ offline, without prior knowledge of which states are more likely, we choose a uniform $Q_{0}$.
In order to assure even coverage of the state space, we chose to sample them from the quasi-random sequence $R_2$, which has been shown empirically to minimize the discrepancy \cite{Martin18el}.
Hence, we perform in practice quasi Monte Carlo method (QMCM) for the computation of $\MImSa_i$ \cite{Caflisch1998an}.
It has been shown in many empirical examples that QMCM obtains faster convergence in practice than regular MC with an equivalent number of samples.
In theory, it is possible to provide a deterministic upper bound to the QMCM integration error via the Koksma-Hlwaka inequality \cite[5.6]{Lemieux2009book}.
However, it is generally hard to compute, and infinite for several very simple functions.

For each sampled $x_{n}^{\Delta}$ we estimate $\DZ$ via observation samples, based on assumption \ref{assum:ModelsSampleQuery}
We perform the following importance sampling estimation w.r.t. $(p_Z+q_Z)\slash 2$:
\begin{gather}
    \DZEst(x_{n}^{\Delta}) = \sum_{j=1}^{N_Z}{2\cdot\frac{\lvert p_Z(z_{j}^{n}\mid x_{n}^{\Delta}) - q_Z(z_{j}^{n}\mid x_{n}^{\Delta}) \rvert}{p_Z(z_{j}^{n}\mid x_{n}^{\Delta}) + q_Z(z_{j}^{n}\mid x_{n}^{\Delta})}}
\end{gather}
where $\{ z_{j}^{n} \}_{n=1}^{N_Z} \overset{\textit{i.i.d.}}{\sim}(p_Z+q_Z)\slash 2$.
It is possible to quantify the MC estimation error from this step, however we assume that with enough offline compute power it could be made negligible, such that $\DZ\approx\DZEst$.

We did several optimizations in order to compute $\MImSa_i$ in a time efficient manner.
The first is by pre-filtering to only keep $x_{n}^{\Delta}$ for which $\DZ(x_{n}^{\Delta}) > \Delta_{\textit{Thresh}}$.
The second optimization is to only consider sampled states within a truncation distance of $d_T$ from state particle $x_{i}^{j}$ in \eqref{eq:StateBoundApprox}.
We implemented this by keeping all $x_{n}^{\Delta}$ in a KD-tree for efficient radius queries \cite{maneewongvatana1999wcg}.
We chose $\Delta_{\textit{Thresh}}$ and $d_T$ such that the error in computing $\MImSa_i$ is at most $V^{\max}\cdotp 10^{-4}$.
Lastly, since the runtime complexity of $\MImSa_i$ grows linearly with $C$, the number of particles in the belief $\bar{b}_i$, we limit the number of particles used to $N_x$ by performing MC estimation w.r.t. the particle belief: $\Estim{\MImSa}_{i}(\bar{b}_i,a)=\frac{1}{N_x}\sum_{j=1}^{N_X}{\MImSa_{i}(x_{i}^{j}, a)}$ where $\{ x_{i}^{j} \}_{j=1}^{N_X} \overset{\textit{i.i.d.}}{\sim} \bar{b}_{i}$.

\subsection{Empirical Bound Evaluation}

\begin{figure}[t]
    \centering
    \includegraphics[width=0.8\columnwidth]{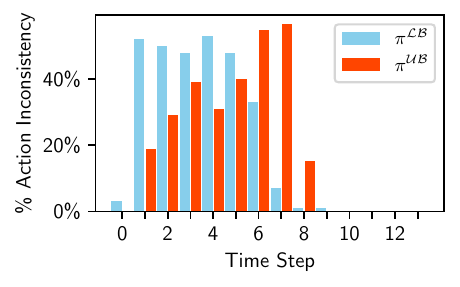}
    \caption{
    Percentage of scenarios in each time step in which the lower or upper bound policies, $\PiLB$ or $\PiUB$, chose an action different from the simplified value policy $\pi^{q_Z}$.
    }
    \label{fig:results_order_inversion}
\end{figure}

In all scenarios, a planning session is fixed at 500 simulations of PFT-DPW with the same parameters, as described in the supplementary material.

We record for each time step the estimated expected value $\Estim{Q}_{t}^{[p_Z\slash q_Z]}$ and the plan session duration.
In simplified planning, we also record the expected action cumulative bound $\PhiImSaEst_{\PBMDP,t}$.
The policy is maximization of the estimated action value function based on the original or simplified model, i.e. $\pi^{[p_Z\slash q_Z]}_{t}\bydef \max_a \Estim{Q}_{t}^{[p_Z\slash q_Z]}(\bar{b}_{t},a)$, and we respectively name them as the original or simplified value policy.
After each planning session, we apply the selected action, update the particle belief according to the observation received, and continue the scenario until the POMDP terminates.

In our first test, we run 100 scenarios of each planning scheme - once with the original observation model, and once with the simplified with the additional computation of $\PhiImSaEst_{\PBMDP,t}$.
As shown in figure \ref{fig:results_times}, the simplified planning time is greatly reduced for the same number of simulations compared to planning with the original model.
These results were expected because of increased overhead for sampling and evaluating the PDF of the original observation model, leading to longer planning times even when additionally computing $\PhiImSaEst$ in the simplified planning.

In our second test, we quantified how often the simplified policy is different from the lower or upper bound policies.
We denote the lower bound policy as $\PiLB_t \bydef \argmax_a \{\Estim{Q}_t^{q_Z}(\PB{t},a) - \PhiImSaEst_t(\PB{t},a)\}$ and the upper bound policy as $\PiUB_t \bydef \argmax_a \{\Estim{Q}_t^{q_Z}(\PB{t},a) + \PhiImSaEst_t(\PB{t},a)\}$.
In our results in figure \ref{fig:results_order_inversion}, we can see that there is a great difference between the policies, in particular in time steps 1-7, which is when the agent is mostly around the light region.
$\PiLB$ tends to steer away from the light region, hence it differs mostly in the earlier time steps, whereas $\PiUB$ prefers the light region, hence it chooses to stay there during the descent of the agent towards the goal.
These results indicate that the bounds are non-trivial, and corresponding policies do point towards different objectives as the scenario progresses.

\section{Conclusion}

This paper builds upon the paradigm of solving POMDP planning problems with simplification for adhering to computational limitations. 
We suggest a planning framework with a simplified observation model, to the end of practically solving POMDPs with complex high dimensional observations like visual observations, with finite time performance guarantees.
We formulate a novel bound based on local reweighting of pre-calculated TV-distances at pre-sampled states, and show that its estimator bounds with high probability the theoretical value function of the original problem.
Finally, an example showcasing how the bounds can influence the decision-making process is presented for both the lower and upper bound policies.
In future research, we envision the use of our bounds during the planning process itself, for pruning of action branches to the end of runtime improvement, or for certifying performance when they're computed explicitly.

\section*{Acknowledgements}
This work was supported by the Israel Science Foundation (ISF) and by US NSF/US-Israel BSF.


\bibliography{refs}

\ifthenelse{\value{arxiv}>0} 
{
    \appendix
    \makeatletter
\newcommand{\labelrpt}[1]{
    \@ifundefined{r@#1}{\label{#1}}{}
}
\makeatother

\section{Proofs}

\subsection{Lemma \ref{lem:StateExptEquivRpt}}

\begin{lemmarpt}
    \label{lem:StateExptEquivRpt}
    The expected belief-dependent reward w.r.t. histories, is equivalent to the expected state-dependent reward w.r.t. the joint distribution of states and observations.
    \begin{gather}
        \ExptFlat{t+1:i}{[{p_Z}\slash{q_Z}]}{r_i(b_{i},\Policy{i})}=\ExptFlat{t+1:i}{p_T,[{p_Z}\slash{q_Z}]}{r_i(x_{i},\Policy{i})}. 
    \label{eq:state_action_reward_rpt}
    \end{gather}
\end{lemmarpt}

\begin{proof}
Without loss of generality we prove for $p_Z$.

From the definition of the expectation, we write it as the following explicit integral:
\begin{gather}
    \ExptFlat{t+1:i}{p_Z}{r_i(b_{i},\Policy{i})} \\
    =\intop_{z_{t+1:i}}p_Z\left(z_{t+1}\mid H_{t+1}^{-}\right)\cdots p_Z\left(z_{i}\mid H_{i}^{-}\right) \nonumber \\ 
    r_i\left(b_{i},\pi_{i}(H_{i})\right) \dif z_{t+1:i}
\end{gather}
From the assumption of a state reward and Fubini's theorem, we replace the belief reward with the expected state reward:
\begin{gather}
    \ExptFlat{t+1:i}{p_Z}{r_i(b_{i},\Policy{i})} \nonumber \\
    =\intop_{z_{t+1:i}}\intop_{x_{i}}p_Z\left(z_{t+1}\mid H_{t+1}^{-}\right)\cdots p_Z\left(z_{i}\mid H_{i}^{-}\right) \nonumber \\
    b_i\left(x_{i}\right)r_i\left(x_{i},\pi_{i}(H_{i})\right)\dif x_{i}\dif z_{t+1:i}
\end{gather}
We apply the following steps of Bayes' rule, and marginalization and chain rule repetitively from time $i$ until time $t$.
\begin{enumerate}
    \item We apply Bayes' rule to the belief $b_i(x_i)$. By definition, $b_i(x_i) \bydef \probd (x_{i}\mid H_{i-1}, \Policy{i-1}(H_{i-1}), z_{i})$, and therefore $b_i(x_i)=\frac{p_Z\left(z_{i}\mid x_{i}\right)}{p_Z\left(z_{i}\mid H_{i}^{-}\right)}b_i^{-}\left(x_{i}\right)$.
    \begin{gather}
        \ExptFlat{t+1:i}{p_Z}{r_i(b_{i},\Policy{i})} \nonumber \\
        =\intop_{z_{t+1:i}}\intop_{x_{i}}p_Z\left(z_{t+1}\mid H_{t+1}^{-}\right)\cdots \cancel{p_Z\left(z_{i}\mid H_{i}^{-}\right)} \nonumber \\
        \frac{p_Z\left(z_{i}\mid x_{i}\right)b_i^{-}\left(x_{i}\right)}{\cancel{p_Z\left(z_{i}\mid H_{i}^{-}\right)}}r_i\left(x_{i},\pi_{i}(H_{i})\right)\dif x_{i}\dif z_{t+1:i}
    \end{gather}
    \item We marginalize over the state $x_{i-1}$ and then apply the chain rule to the propagated belief $b_i^{-}(x_i)$. By definition the propagated belief satisfies $b_i^{-}(x_i)\bydef \probd (x_i \mid H_{i-1}, \Policy{i-1}(H_{i-1}))$, and therefore it holds that $b_i^{-}(x_i)=\int_{x_{i-1}} b_{i-1}\left(x_{i-1}\right)p_T\left(x_{i}\mid x_{i-1},\pi_{i-1}\left(H_{i-1}\right)\right)$.
    \begin{gather}
        \ExptFlat{t+1:i}{p_Z}{r_i(b_{i},\Policy{i})} \nonumber \\
        =\intop_{z_{t+1:i}}\intop_{x_{i-1:i}}p_Z\left(z_{t+1}\mid H_{t+1}^{-}\right)\cdots p_Z\left(z_{i-1}\mid H_{i-1}^{-}\right) \nonumber \\
        b_{i-1}\left(x_{i-1}\right)p_T\left(x_{i}\mid x_{i-1},\pi_{i-1}\left(H_{i-1}\right)\right)p_Z\left(z_{i}\mid x_{i}\right) \nonumber \\
        r_i\left(x_{i},\pi_{i}(H_{i})\right)\dif x_{i-1:i}\dif z_{t+1:i}
    \end{gather}
\end{enumerate}
By repeating steps 1 and 2 until marginalizing over the state $x_t$, we conclude that:
\begin{gather}
    \ExptFlat{t+1:i}{p_Z}{r_i(b_{i},\Policy{i})} \nonumber \\
    =\intop_{z_{t+1:i}}\intop_{x_{t:i}}b_t\left(x_{t}\right)\prod_{j=t+1}^{i}p_T\left(x_{j}\mid x_{j-1},\pi_{j-1}\left(H_{j-1}\right)\right) \nonumber \\
    \cdot\prod_{k=t+1}^{i}p_Z\left(z_{k}\mid x_{k}\right)\cdot r_i\left(x_{i},\pi_{i}(H_{i})\right)\dif x_{t:i}\dif z_{t+1:i} \nonumber \\
    =\ExptFlat{t+1:i}{p_T,p_Z}{r_i(x_{i},\Policy{i})}
\end{gather}
\end{proof}

\subsection{Theorem \ref{thm:AltModelThBoundRpt}}

\begin{theoremrpt}
    \label{thm:AltModelThBoundRpt}
    Assume current belief is $b_{t}$ and a given policy is $\pi$.
    Denote with $b_{i}^{p_Z}$, $b_{i}^{q_Z}$ the future belief at time step $i$ updated with either $p_Z$ or $q_Z$, respectively. Then it holds that
    \begin{gather}
        \textstyle \absvalflat{\ExptFlatOrigHistory{i}{r_{i}(b_{i}^{p_Z},\Policy{i})}-\ExptFlatSimpHistory{i}{r_{i}(b_{i}^{q_Z},\Policy{i})}} \nonumber \\
        \textstyle \leq \Rmax{i}\sum_{l=t+1}^{i}\ExptFlatSimpState{l^-}{\DZ(x_{l})},
        \label{eq:state_bound_reward_rpt}
    \end{gather}
    i.e. the difference between the expected reward at future time $i$ for the original and simplified POMDP is bounded by the maximum reward, times the sum of the expected state-dependent TV-distances between the observation models.
\end{theoremrpt}

\begin{proof}
We note that the two expectations share integration domain. By the linearity of the integral, we can combine integrands that are equal across the domains, and use the distributive property of the integral. We then apply the integral triangle inequality to the difference term of the two different observation models.
\begin{gather}
    \absvalflat{\ExptFlatOrigHistory{i}{r_{i}(b_{i}^{p_Z},\Policy{i})}-\ExptFlatSimpHistory{i}{r_{i}(b_{i}^{q_Z},\Policy{i})}}\\
    =\absvalflat{\ExptFlatOrigState{i}{r_{i}(x_{i},\Policy{i})}-\ExptFlatSimpState{i}{r_{i}(x_{i},\Policy{i})}}\\
    =\left|\intop_{z_{t+1:i}}\intop_{x_{t:i}} b_{t}(x_{t})\prod_{j=t+1}^{i}p_{T}\left(x_{j}\mid x_{j-1},\pi_{j-1}\left(H_{j-1}\right)\right) \nonumber \right. \\
    \left(\prod_{j=t+1}^{i}p_{Z}\left(z_{j}\mid x_{j}\right)-\prod_{j=t+1}^{i}q_{Z}\left(z_{j}\mid x_{j}\right)\right) \nonumber \\
    \cdot r(x_{i},\pi_i(H_{i}))\dif x_{t:i}\dif z_{t+1:i}\Bigg|\\
    \leq\intop_{z_{t+1:i}}\intop_{x_{t:i}}b_{t}(x_{t}) \prod_{j=t+1}^{i}p_{T}\left(x_{j}\mid x_{j-1},\pi_{j-1}\left(H_{j-1}\right)\right) \nonumber \\
    \left|\prod_{j=t+1}^{i}p_{Z}\left(z_{j}\mid x_{j}\right)-\prod_{j=t+1}^{i}q_{Z}\left(z_{j}\mid x_{j}\right)\right| \nonumber \\ 
    \cdot\left|r(x_{i},\pi_i(H_{i}))\right|\dif x_{t:i}\dif z_{t+1:i}\
\end{gather}
We utilize the assumption of a bounded state reward, of the form $\left|r(x_{i},\pi_i(H_{i}))\right|\leq \Rmax{i}$.
\begin{gather}
    \leq \Rmax{i} \cdot\intop_{z_{t+1:i}}\intop_{x_{t:i}}b_{t}(x_{t}) \nonumber \\
    \prod_{j=t+1}^{i}p_T\left(x_{j}\mid x_{j-1},\pi_{j-1}\left(H_{j-1}\right)\right) \nonumber \\
    \cdot\left|\prod_{j=t+1}^{i}p_{Z}\left(z_{j}\mid x_{j}\right)-\prod_{j=t+1}^{i}q_{Z}\left(z_{j}\mid x_{j}\right)\right|\dif x_{t:i}\dif z_{t+1:i}
\end{gather}
Next, we add and subtract a mixed term of $i-t-1$ simplified observation factors $q_Z$ and a single original observation factor $p_Z$. This factor is carefully chosen to pick together common factors such that we'll be left with our desired expectations - computed over the simplified observation model.
\begin{gather}
    =\Rmax{i}\cdot\intop_{z_{t+1:i}}\intop_{x_{t:i}}b_{t}(x_{t}) \nonumber \\
    \prod_{j=t+1}^{i}p_{T}\left(x_{j}\mid x_{j-1},\pi_{j-1}\left(H_{j-1}\right)\right)\\ 
    \cdot \left|\prod_{j=t+1}^{i}p_{Z}\left(z_{j}\mid x_{j}\right)-\prod_{j=t+1}^{i-1}q_{Z}\left(z_{j}\mid x_{j}\right)p_{Z}\left(z_{i}\mid x_{i}\right) \right. \nonumber \\
    \left. +\prod_{j=t+1}^{i-1}q_{Z}\left(z_{j}\mid x_{j}\right)p_{Z}\left(z_{i}\mid x_{i}\right)-\prod_{j=t+1}^{i}q_{Z}\left(z_{j}\mid x_{j}\right)\right| \nonumber \\
    \dif x_{t:i}\dif z_{t+1:i}\\
    =\Rmax{i}\cdot\intop_{z_{t+1:i}}\intop_{x_{t:i}}b_{t}(x_{t}) \nonumber \\
    \prod_{j=t+1}^{i}p_{T}\left(x_{j}\mid x_{j-1},\pi_{j-1}\left(H_{j-1}\right)\right) \\
    \cdot \left|\prod_{j=t+1}^{i-1}q_{Z}\left(z_{j}\mid x_{j}\right)\left(p_{Z}\left(z_{i}\mid x_{i}\right)-q_{Z}\left(z_{i}\mid x_{i}\right)\right)+ \right. \nonumber \\
    \left. p_{Z}\left(z_{i}\mid x_{i}\right)\left(\prod_{j=t+1}^{i-1}p_{Z}\left(z_{j}\mid x_{j}\right)-\prod_{j=t+1}^{i-1}q_{Z}\left(z_{j}\mid x_{j}\right)\right)\right| \nonumber \\
    \dif x_{t:i}\dif z_{t+1:i}
\end{gather}
We once more use the triangle inequality and the linearity of integral to separate the obtained integral into two. We denote these terms as $(1)$ and $(2)$ and then handle each separately.
\begin{gather}
    \leq \Rmax{i} \nonumber \\
    \cdotp (1)\bydef \left\{
        \begin{aligned}
        \intop_{z_{t+1:i}}\intop_{x_{t:i}}b_{t}(x_{t})\prod_{j=t+1}^{i}p_{T}\left(x_{j}\mid x_{j-1},\pi_{j-1}\left(H_{j-1}\right)\right) \\
        \cdot\prod_{j=t+1}^{i-1}q_{Z}\left(z_{j}\mid x_{j}\right)\left|p_{Z}\left(z_{i}\mid x_{i}\right)-q_{Z}\left(z_{i}\mid x_{i}\right)\right| \\
        \dif x_{t:i}\dif z_{t+1:i}
        \end{aligned}
        \right. \\
    +\Rmax{i} \nonumber \\
    \cdot (2) \bydef \left\{
        \begin{aligned}
        \intop_{z_{t+1:i}}\intop_{x_{t:i}}b_{t}(x_{t})\prod_{j=t+1}^{i}p_{T}\left(x_{j}\mid x_{j-1},\pi_{j-1}\left(H_{j-1}\right)\right) \\
        \cdot p_{Z}\left(z_{i}\mid x_{i}\right)\left|\prod_{j=t+1}^{i-1}p_{Z}\left(z_{j}\mid x_{j}\right)-\prod_{j=t+1}^{i-1}q_{Z}\left(z_{j}\mid x_{j}\right)\right| \\
        \dif x_{t:i}\dif z_{t+1:i}
        \end{aligned}
        \right.
\end{gather}
The first term turns out to become the expected TV-distance w.r.t. trajectories taken with the simplified observation model.
\begin{gather}
    (1)=\intop_{z_{t+1:i-1}}\intop_{x_{t:i}}b_{t}(x_{t})\prod_{j=t+1}^{i}p_{T}\left(x_{j}\mid x_{j-1},\pi_{j-1}\left(H_{j-1}\right)\right) \nonumber \\
    \cdot \prod_{j=t+1}^{i-1}p_{Z}\left(z_{j}\mid x_{j}\right)\left(\intop_{z_{i}}\left|p_{Z}\left(z_{i}\mid x_{i}\right)-q_{Z}\left(z_{i}\mid x_{i}\right)\right|\dif z_{i}\right)\nonumber \\
    \dif x_{t:i}\dif z_{t+1:i} \nonumber \\
    =\Rmax{i}\cdot \ExptFlatSimpState{i^-}{\DZ(x_{i})}
\end{gather}
The second term is the recursive term, in which we can integrate out first the last observation factor, and then the last transition factor.
\begin{gather}
    (2)=\intop_{z_{t+1:i-1}}\intop_{x_{t:i}}b_{t}(x_{t})\prod_{j=t+1}^{i}p_{T}\left(x_{j}\mid x_{j-1},\pi_{j-1}\left(H_{j-1}\right)\right) \nonumber \\
    \cdot \left|\prod_{j=t+1}^{i-1}p_{Z}\left(z_{j}\mid x_{j}\right)-\prod_{j=t+1}^{i-1}q_{Z}\left(z_{j}\mid x_{j}\right)\right| \nonumber \\
    \cancel{\left(\intop_{z_{i}}p_{Z}\left(z_{i}\mid x_{i}\right)\dif z_{i-1}\right)}\dif x_{t:i}\dif z_{t+1:i}\\
    =\intop_{z_{t+1:i-1}}\intop_{x_{t:i}}b_{t}(x_{t})\prod_{j=t+1}^{i-1}p_{T}\left(x_{j}\mid x_{j-1},\pi_{j-1}\left(H_{j-1}\right)\right) \nonumber \\
    \cdot \left|\prod_{j=t+1}^{i-1}p_{Z}\left(z_{j}\mid x_{j}\right)-\prod_{j=t+1}^{i-1}q_{Z}\left(z_{j}\mid x_{j}\right)\right| \nonumber \\
    \cancel{\left(\int_{x_{i}}p\left(x_{i}\mid x_{i-1},\pi_{i-1}\left(z_{i-1}\right)\right)\dif x_{i}\right)}\dif x_{t:i-1}\dif z_{t+1:i}\\
    =\intop_{z_{t+1:i}}\intop_{x_{t:i-1}}b_{t}(x_{t})\prod_{j=t+1}^{i-1}p\left(x_{j}\mid x_{j-1},\pi_{j-1}\left(H_{j-1}\right)\right) \nonumber \\
    \cdot\left|\prod_{j=t+1}^{i-1}p_{Z}\left(z_{j}\mid x_{j}\right)-\prod_{j=t+1}^{i-1}q_{Z}\left(z_{j}\mid x_{j}\right)\right|\dif x_{t:i-1}\dif z_{t+1:i-1}
\end{gather}
Notice the recursive relationship:
\begin{gather}
    \intop_{z_{t+1:i}}\intop_{x_{t:i}}b_{t}(x_{t})\prod_{j=t+1}^{i}p\left(x_{j}\mid x_{j-1},\pi_{j-1}\left(H_{j-1}\right)\right) \nonumber \\
    \cdot\left|\prod_{j=t+1}^{i}p_{Z}\left(z_{j}\mid x_{j}\right)-\prod_{j=t+1}^{i}q_{Z}\left(z_{j}\mid x_{j}\right)\right|\dif x_{t:i}\dif z_{t+1:i}\\
    \leq\ExptFlatSimpState{i^-}{\DZ(x_{i})} \nonumber \\
    +\intop_{z_{t+1:i}}\intop_{x_{t:i-1}}b_{t}(x_{t})\prod_{j=t+1}^{i-1}p\left(x_{j}\mid x_{j-1},\pi_{j-1}\left(H_{j-1}\right)\right) \nonumber \\
    \cdot\left|\prod_{j=t+1}^{i-1}p_{Z}\left(z_{j}\mid x_{j}\right)-\prod_{j=t+1}^{i-1}q_{Z}\left(z_{j}\mid x_{j}\right)\right|\dif x_{t:i-1}\dif z_{t+1:i-1}
\end{gather}

In the base case of $i=t+1$ we get directly
\begin{gather}
\intop_{z_{t+1:i}}\intop_{x_{t:i}}b_{t}(x_{t})\prod_{j=t+1}^{i}p\left(x_{j}\mid x_{j-1},\pi_{j-1}\left(H_{j-1}\right)\right)\nonumber \\
\cdot\left|\prod_{j=t+1}^{i}p_{Z}\left(z_{j}\mid x_{j}\right)-\prod_{j=t+1}^{i}q_{Z}\left(z_{j}\mid x_{j}\right)\right|\dif x_{t:i}\dif z_{t+1:i}\\
=\ExptFlatSimpState{i^-}{\DZ(x_{i})} =\ExptFlat{b_{t}}{}{\ExptFlat{t+1}{p_T}{\DZ(x_{t+1})}}
\end{gather}
Hence we can see that the sum of the recursion turns out to be
\begin{gather}
    \absvalflat{\ExptFlatOrigHistory{i}{r_{i}(b_{i}^{p_Z},\Policy{i})}-\ExptFlatSimpHistory{i}{r_{i}(b_{i}^{q_Z},\Policy{i})}} \\
    \leq \Rmax{i}\sum_{l=t+1}^{i}\ExptFlatSimpState{l^-}{\DZ(x_{l})}
\end{gather}
\end{proof}

\subsection{Corollary \ref{crl:TheoreticalBoundRpt}}

\begin{corollaryrpt}
    \label{crl:TheoreticalBoundRpt}
    The difference between the original and simplified value functions can be bounded by the following sum of scaled expected TV-distance terms:
    \begin{gather}
        \textstyle \absvalflat{V_{t}^{\pi,p_Z}(b_{t})-V_{t}^{\pi,q_Z}(b_{t})}\nonumber\\
        \textstyle \leq\sum_{i=t+1}^{L}\ExptFlatSimpState{i^-}{\DZ(x_{i})}\sum_{l=i}^{L}\gamma^{l-t}\Rmax{l} \\
        \textstyle =\sum_{i=t+1}^{L}\Vmax{i}\cdotp\ExptFlatSimpState{i^-}{\DZ(x_{i})}. 
        \label{eq:bound_value_rpt}
    \end{gather}
\end{corollaryrpt}

\begin{proof}
We substitute for the definition of $V_{t}^\pi$, and we can see that the immediate reward at time $i$ cancels out.
\begin{gather}
    \absval{V_{t}^{\pi,p_Z}(b_{t})-V_{t}^{\pi,q_Z}(b_{t})} \\
    = \left| \ExptFlatOrigHistory{L}{\sum_{i=t}^{L}\gamma^{i-t}r_{i}(b_{i},\Policy{i})} \right. \nonumber \\
    \left. - \ExptFlatSimpHistory{L}{\sum_{i=t}^{L}\gamma^{i-t}r_{i}(b_{i},\Policy{i})} \right| \\
    = \left| \cancel{r_{t}(b_{t},\Policy{t})} + \ExptFlatOrigHistory{L}{\sum_{i=t+1}^{L}\gamma^{i-t}r_{i}(b_{i},\Policy{i})} \right. \nonumber \\
    \left. - \cancel{r_{t}(b_{t},\Policy{t})} - \ExptFlatSimpHistory{L}{\sum_{i=t+1}^{L}\gamma^{i-t}r_{i}(b_{i},\Policy{i})} \right|
\end{gather}
We apply Lemma \ref{lem:StateExptEquivRpt} to substitute expectation over belief rewards to expectation over state rewards, and we rearrange terms so that we get summation over difference terms of expected state rewards.
\begin{gather}
    = \left| \sum_{i=t+1}^{L}\gamma^{i-t}\ExptFlatOrigState{i}{r_{i}(x_{i},\Policy{i})} \right. \nonumber \\ 
    \left. - \sum_{i=t+1}^{L}\gamma^{i-t}\ExptFlatSimpState{i}{r_{i}(x_{i},\Policy{i})} \right| \\
    = \left| \sum_{i=t+1}^{L}\gamma^{i-t}\left(\ExptFlatOrigState{i}{r_{i}(x_{i},\Policy{i})} - \ExptFlatSimpState{i}{r_{i}(x_{i},\Policy{i})}\right) \right|
\end{gather}
We now use the triangle inequality to separate into $L-t$ summands on which we can apply Theorem \ref{thm:AltModelThBoundRpt}:
\begin{gather}
    \leq \sum_{i=t+1}^{L}\gamma^{i-t} \left| \ExptFlatOrigState{i}{r_{i}(x_{i},\Policy{i})} - \ExptFlatSimpState{i}{r_{i}(x_{i},\Policy{i})} \right| \\
    \stackrel{\textit{Thm. \ref{thm:AltModelThBoundRpt}}}{\leq} \sum_{i=t+1}^{L}\gamma^{i-t} \Rmax{i}\sum_{l=t+1}^{i}\ExptFlatSimpState{l^-}{\DZ(x_{l})}
\end{gather}
We switch the order of summation such that for every expected-$\DZ$ term we sum all of its $\Rmax{}$ terms:
\begin{gather}
    = \gamma \Rmax{t+1} \left( \ExptFlatSimpState{{t+1}^-}{\DZ(x_{t+1})} \right) \nonumber \\
    + \gamma^2 \Rmax{t+2} \left( \ExptFlatSimpState{{t+1}^-}{\DZ(x_{t+1})} 
    \right. \nonumber \\
    \left. 
    + \ExptFlatSimpState{{t+2}^-}{\DZ(x_{t+2})} \right) \nonumber \\
    + \dots \nonumber \\
    \begin{multlined}
        + \gamma^{L-t} \Rmax{L} \left( \ExptFlatSimpState{{t+1}^-}{\DZ(x_{t+1})} \right. \\
        \left. + \dots + \ExptFlatSimpState{{t+L}^-}{\DZ(x_{t+L})} \right)
    \end{multlined} \\
    \begin{multlined}
    = \ExptFlatSimpState{{t+1}^-}{\DZ(x_{t+1})} \left( \gamma \Rmax{t+1} + \dots + \gamma^{L-t} \Rmax{L} \right) \\
    + \ExptFlatSimpState{{t+2}^-}{\DZ(x_{t+2})} \left( \gamma^2 \Rmax{t+2} + \dots + \gamma^{L-t} \Rmax{L} \right) \\
    + \dots \\
    + \ExptFlatSimpState{{t+L}^-}{\DZ(x_{t+L})} \gamma^{L-t} \Rmax{L}
    \end{multlined} \\
    =\sum_{i=t+1}^{L}\ExptFlatSimpState{i^-}{\DZ(x_{i})}\sum_{l=i}^{L}\gamma^{l-t}\Rmax{l}
\end{gather}
By identifying the sum of discounted maximum rewards as the maximum value, we arrive at the required:
\begin{equation}
    =\sum_{i=t+1}^{L}\Vmax{i}\cdotp\ExptFlatSimpState{i^-}{\DZ(x_{i})}.
\end{equation}

\end{proof}

\subsection{Theorem \ref{thm:LocalActionBoundRpt}}

\begin{theoremrpt}
    \label{thm:LocalActionBoundRpt}
    Under the conditions of Theorem \ref{thm:AltModelThBound}, the difference between the original and simplified value function can be bounded by
    \begin{gather}
        \textstyle \absvalflat{V_{t}^{\pi,p_Z}(b_{t})-V_{t}^{\pi,q_Z}(b_{t})}\leq M_{t}^{\pi}(b_{t}). 
        \label{eq:MBoundV_rpt}
    \end{gather}
    In addition, the respective difference in action value function can be bounded by the action cumulative bound,
    \begin{gather}
        \textstyle \absvalflat{Q_{t}^{\pi,p_Z}(b_{t},a)-Q_{t}^{\pi,q_Z}(b_{t},a)}\leq \Phi_{t}^{\pi}(b_{t},a). 
        \label{eq:PhiBoundQ_rpt}
    \end{gather}
\end{theoremrpt}

\begin{proof}
First we prove for $\eqref{eq:MBoundV_rpt}$:
\begin{gather}
    \textstyle M_{t}^{\Policy{}}(b_t) = \ExptFlatSimpHistory{L-1}{\sum_{i=t}^{L-1}m_{i}(b_{i},\Policy{i})} \\
    =\textstyle \sum_{i=t+1}^{L} \ExptFlatSimpHistory{i-1}{m_{i-1}(b_{i-1},\Policy{i-1})} \\
    \textstyle \overset{\textit{Lem. \ref{lem:StateExptEquivRpt}}}{=}\sum_{i=t+1}^{L} \ExptFlatSimpState{i-1}{m_{i-1}(x_{i-1},\Policy{i-1})} \\
    \textstyle =\sum_{i=t+1}^{L}\Vmax{i}\cdot\ExptFlatSimpState{i-1}{\ExptFlat{i}{p_T}{\DZ(x_{i})}}  \\
    \textstyle =\sum_{i=t+1}^{L} \Vmax{i}\cdot\ExptFlatSimpState{i^{-}}{\DZ(x_{i})}
\end{gather}
And the proof is clear by seeing that this is the exact term that bounds $\eqref{eq:bound_value_rpt}$.
For proving $\eqref{eq:PhiBoundQ_rpt}$, note that for every action $a$ we can define a policy $\Policy{}^{a}$ by performing $a$ and then continuing with $\Policy{}$.
For every policy $\Policy{}$ it holds that $Q_{t}^{\Policy{}}(b_t, \Policy{t}) = V_{t}^{\Policy{}}(b_t)$, and specifically for $\Policy{}^{a}$.
Therefore:
\begin{gather}
    \textstyle \absval{Q_{t}^{\pi,p_Z}(b_{t},a)-Q_{t}^{\pi,q_Z}(b_{t},a)}= \\ 
    \textstyle \absval{V_{t}^{\Policy{}^{a},p_Z}(b_{t})-V_{t}^{\Policy{}^{a},q_Z}(b_{t})} \\
    \leq M_{t}^{\Policy{}^{a}}(b_t) = \Phi_{t}^{\pi}(b_{t},a).
\end{gather}
\end{proof}

\subsection{Theorem \ref{thm:ProbConvergeRpt}}

The proof of Theorem \ref{thm:ProbConvergeRpt} is an adaptation of Lemma 2 from \cite{Lim22arxiv}, with relatively small modifications.

We first rely on the Particle Likelihood SN Estimator Convergence lemma from \cite{Lim22arxiv} (originally Lemma 1):

\begin{lemmarpt}[Particle Likelihood SN Estimator Convergence] 
    \label{lem:SNBounds}
     Suppose a function f is bounded by a finite constant $\left\Vert f\right\Vert _{\infty}\leq f_{\max}$, and a particle belief state $\PB{t}=\{x_{t}^{i}, w_{t}^{i}\}_{i=1}^{C}$ at depth $t$ that represents $\PB{t}$ with particle likelihood weighting that is recursively updated as $w_{t}^{i}=w_{t-1}^{i}\cdot p_{Z}\left(z\mid x_{t+1}\right)$ for an observation sequence $\{z_{n}^{i}\}_{n=1}^{t}$.
     Then, for all $t=0,\dots,L$, the following weighted average is the SN estimator of $f$ under the belief $b_{t}$ corresponding to observation sequence $\{ z_{n}\} _{n=1}^{t}:\tilde{\mu}_{\PB{t}}[f]=\frac{\sum_{i=1}^{C}w_{t}^{i}f(x_{t}^{i})}{\sum_{i=1}^{C}w_{t}^{i}}$ and the following concentration bound holds with probability at least $1-3\exp\left(-C\cdot k_{\max}^{2}\left(\lambda,C\right)\right)$
     \begin{gather}
        \left|\mathbb{E}_{s\sim b_{t}}\left[f\left(x\right)\right]-\tilde{\mu}_{\bar{b}_{t}}\left[f\right]\right|\leq\lambda \\
        k_{\max}\left(\lambda,C\right)\triangleq\frac{\lambda}{f_{\max}d_{\infty}^{\max}}-\frac{1}{\sqrt{C}} \\
        d_{\infty}\left(\mathcal{P}^{t}||\mathcal{Q}^{t}\right)=\esssup_{x\sim\mathcal{Q}^{t}} w_{\mathcal{P}^{t}\slash\mathcal{Q}^{t}}\left(x\right)\leq d_{\infty}^{\max}
     \end{gather}
     where $\mathcal{P}^{t}$ is the target distribution and $\mathcal{Q}^{t}$ is the distribution of the particle filter.
\end{lemmarpt}

For the next lemma, we define a new theoretical algorithm Sparse Sampling-$\omega$-$\pi$, shown in Algorithm \ref{alg:SparseSamplingOmegaPi}. The algorithm estimates the value function of a policy $\pi$ according to Sparse Sampling-$\omega$ \cite{Lim23jair}.

\begin{algorithm}[tb]
    \caption{Sparse Sampling-$\omega$-$\pi$}
    \label{alg:SparseSamplingOmegaPi}
    \textbf{Global Variables}: $\gamma,C,L,\pi$\\
    \textbf{Algorithm:} $\text{Estimate}V^{\pi}(\bar{b},t)$.\\
    \textbf{Input}: Particle belief set $\bar{b}=\left\{ \left(x_{i},w_{i}\right)\right\}$, depth $t$, policy $\pi$. \\
    \textbf{Output}: A scalar $\hat{V}_{\omega,t}^{\pi}\left(\bar{b}_{t}\right)$ that is an estimate of $V_{t}^{\pi}\left(\bar{b}_{t}\right)$.
    \begin{algorithmic}[1]
    \IF {$t\geq L$}
    \STATE \textbf{return} 0
    \ENDIF
    \begingroup
    \color{blue}
    \STATE $\hat{Q}_{t}^{\pi}(\bar{b},\pi(\bar{b}))\leftarrow\text{Estimate}Q^{\pi}(\bar{b},\pi(\bar{b}),t)$ \label{alg:BlueVariation}
    \endgroup
    \begingroup
    \color{red}
    \FORALL {$a\in\mathcal{A}$} \label{alg:RedVariationStart}
    \STATE $\hat{Q}_{t}^{\pi}(\bar{b},a)\leftarrow\text{Estimate}Q^{\pi}(\bar{b},a,t)$ \label{alg:RedVariationEnd}
    \ENDFOR
    \endgroup
    \STATE \textbf{return} $ \hat{V}_{t}^{\pi}(\bar{b})\leftarrow\hat{Q}_{t}^{\pi}(\bar{t},\pi(\bar{b}))$
    \end{algorithmic}
    \textbf{Algorithm:} $\text{Estimate}Q^{\pi}\left(\bar{b},a,t\right)$.\\
    \textbf{Input}: Particle belief set $\bar{b}=\left\{ \left(x_{i},w_{i}\right)\right\}$, action $a$, depth $t$, policy $\pi$. \\
    \textbf{Output}: A scalar $\hat{Q}_{\omega,t}^{\pi}\left(\bar{b},a\right)$ that is an estimate of $Q_{t}^{\pi}\left(b,a\right)$.
    \begin{algorithmic}[1]
    \FORALL {$i=1,\dots,C$}
    \STATE $\bar{b}_{i}^{\prime},\rho\leftarrow\text{GenPF}(\bar{b},a)$
    \STATE $\hat{V}_{t+1}^{\pi}(\bar{b}_{i}^{\prime})\leftarrow\text{Estimate}V^{\pi}(\bar{b}_{i}^{\prime},t+1)$
    \ENDFOR
    \STATE \textbf{return} $\hat{Q}_{t}^{\pi}(\bar{b},a)\leftarrow\rho+\frac{1}{C}\sum_{i=1}^{C}\gamma\cdot\hat{V}_{t+1}^{\pi}(\bar{b}_{i}^{\prime})$
    \end{algorithmic}
\end{algorithm}

The algorithm has two variations, indicated with the blue line in \ref{alg:BlueVariation}, or with the red lines in \ref{alg:RedVariationStart}-\ref{alg:RedVariationEnd}. The blue variation is for estimating the value of a single policy $\pi$, whereas the red variation expands the entire action space (if it is finite), to estimate the values of each possible policy via the $Q$-value.

\begin{lemmarpt}[Sparse Sampling-$\omega$-$\pi$ $Q$-Value Coupled Convergence] 
    \label{lem:SparseSamplingOmegaPiConvergence}
For a given policy $\pi$, for all $t=0,\dots,L$ and actions $a$, the following bounds hold with probability at least $1-5\left(4C\right)^{L+1}\left(\exp\left(-C\cdot\acute{k}^{2}\right)+\delta_r\left(\nu,N_r\right)\right)$:
\begin{gather}
    \left|Q_{\mathbf{P},t}^{\pi}\left(b_{t},a\right)-\hat{\tilde{Q}}_{\omega,t}^{\pi}\left(\bar{b}_{t},a\right)\right|\leq\alpha_{t}, \\
    \alpha_{t}=\lambda+\nu+\gamma\alpha_{t+1},\ \alpha_{L}=\lambda+\nu, \\
    \left|Q_{\mathbf{M_{P}},t}^{\pi}\left(\bar{b}_{t},a\right)-\hat{\tilde{Q}}_{\omega,t}^{\pi}\left(\bar{b}_{t},a\right)\right|\leq\beta_{t},\\
    \beta_{t}=\nu+\gamma\left(\lambda+\beta_{t+1}\right),\beta_{L}=\nu, \\
    k_{\max}\left(\lambda,C\right)=\frac{\lambda}{4V_{\max}d_{\infty}^{\max}}-\frac{1}{\sqrt{C}},\\
    \acute{k}=\min\left\{ k_{\max},\lambda\slash4\sqrt{2}V_{\max}\right\} 
\end{gather}

Under the assumption that the immediate reward estimate is probabilistically bounded such that $\probd(\left|r_{t}^{i}-\tilde{r}_{t}^{i}\right|\geq\nu)\leq\delta_r(\nu,N_r)$, for a number of samples parameter $N_r$.

If we require the bound to hold for all possible policies that can be extracted from a given belief tree simultaneously, then the probability becomes at least $1-5\left(4\left|A\right|C\right)^{L+1}\left(\exp\left(-C\cdot\acute{k}^{2}\right)+\delta_r\left(\nu,N_r\right)\right)$.
\end{lemmarpt}

\begin{proof}
We prove for observation model $p_Z$ without loss of generality.
\paragraph{POMDP Value Convergence}

We split the difference between the SN estimator and $Q_{\mathbf{P},t}^{\pi}$
into two terms, the reward estimation error $(A)$ and the next-step
value estimation error $(B)$:

\begin{gather}
\left|Q_{\mathbf{P},t}^{\pi}\left(b_{t},a\right)-\hat{\tilde{Q}}_{\omega,t}^{\pi}\left(\bar{b}_{t},a\right)\right|\\
\leq\underbrace{\left|\mathbb{E}_{\mathbf{P}}\left[R\left(x_{t},a\right)\mid b_{t}\right]-\frac{\sum_{i=1}^{C}w_{t}^{i}\tilde{r}_{t}^{i}}{\sum_{i=1}^{C}w_{t}^{i}}\right|}_{(A)}\\
+\gamma\underbrace{\left|\mathbb{E}_{\mathbf{P}}\left[V_{\mathbf{P},t+1}^{\pi}\left(b_{t}az\right)\mid b_{t}\right]-\frac{1}{C}\sum_{i=1}^{C}\hat{\tilde{V}}_{\omega,t+1}^{\pi}\left(\bar{b}_{t+1}^{\prime\left[I_{i}\right]}\right)\right|}_{(\mathrm{B})}
\end{gather}

Where $I_{i}$ is a RV sampled from the probability mass $p_{w,t}\left(I=i\right)=\left(w_{t}^{i}\slash\sum_{j}w_{t}^{j}\right)$, 
and the particle belief $\bar{b}_{t+1}^{\prime\left[I_{i}\right]}$
is updated with an observation generated from $x_{t}^{I_{i}}$.

To prove the base case $t=L$, note that only term $(A)$ is needed
to be bounded, since $t=L$ corresponds to the leaf node of Sparse
Sampling-$\omega$-$\pi$ and no further next step value estimation
is performed.

We split term $(A)$ into two terms:

\begin{gather}
\underbrace{\left|\mathbb{E}_{\mathbf{P}}\left[R\left(x_{t},a\right)\mid b_{t}\right]-\frac{\sum_{i=1}^{C}w_{t}^{i}\tilde{r}_{t}^{i}}{\sum_{i=1}^{C}w_{t}^{i}}\right|}_{(A)}\\
\leq\underbrace{\left|\mathbb{E}_{\mathbf{P}}\left[R\left(x_{t},a\right)\mid b_{t}\right]-\frac{\sum_{i=1}^{C}w_{t}^{i}r_{t}^{i}}{\sum_{i=1}^{C}w_{t}^{i}}\right|}_{(1)\text{Importance sampling error}}\\
+\underbrace{\left|\frac{\sum_{i=1}^{C}w_{t}^{i}r_{t}^{i}}{\sum_{i=1}^{C}w_{t}^{i}}-\frac{\sum_{i=1}^{C}w_{t}^{i}\tilde{r}_{t}^{i}}{\sum_{i=1}^{C}w_{t}^{i}}\right|}_{(2)\text{Reward approximation error}}
\end{gather}

Term $(1)$ is a particle likelihood weighted average of $R\left(\cdotp,a\right)$,
and we will use the SN concentration bounds from Lemma \ref{lem:SNBounds}.
We bound $R$ with $R_{\max}$ and augment $\lambda$ to $\frac{R_{\max}}{4V_{\max}}\lambda$,
in order to obtain the same uniform $t_{\max}$ factor with the following
terms. This also covers the base case since $\frac{R_{\max}}{4V_{\max}}\lambda\leq\lambda=\alpha_{L}$.
Hence the bound will hold with probability at least $1-3\exp\left(-C\cdot t_{\max}^{2}\left(\lambda,C\right)\right)$.

For term $(2)$, we use the triangle inequality to bound the apply
the assumption of a probabilistic bound on the state reward to a probabilistic
bound on the belief reward. Next we use the monotonicity of the weighted
mean, in the context that if all terms have an upper bound $\nu$,
then follows that the average itself is upper bounded by $\nu$. Finally
we use the inverse of the union bound (Boole's inequality) to lower
bound with assumed bound for each individual reward term. 
\begin{gather}
\probd\left(\left|\frac{\sum_{i=1}^{C}w_{t}^{i}r_{t}^{i}}{\sum_{i=1}^{C}w_{t}^{i}}-\frac{\sum_{i=1}^{C}w_{t}^{i}\tilde{r}_{t}^{i}}{\sum_{i=1}^{C}w_{t}^{i}}\right|\leq\nu\right)\\
\overset{\text{Triangle inequality}}{\geq}\\
\probd\left(\left(\sum_{i=1}^{C}w_{t}^{i}\right)^{-1}\sum_{i=1}^{C}w_{t}^{i}\left|\left(r_{t}^{i}-\tilde{r}_{t}^{i}\right)\right|\leq\nu\right)\\
\overset{\text{Weighted mean monotonicity}}{\geq}\probd\left(\bigcap_{i=1}^{C}\left|\left(r_{t}^{i}-\tilde{r}_{t}^{i}\right)\right|\leq\nu\right)\\
\overset{\text{Union bound}}{\geq}1-\sum_{i=1}^{C}\probd\left(\left|\left(r_{t}^{i}-\tilde{r}_{t}^{i}\right)\right|\geq\nu\right)\\
\overset{\text{Assumption}}{\geq}1-\sum_{i=1}^{C}\delta_r\left(\nu,N_r\right)=1-C\delta_r\left(\nu,N_r\right)
\end{gather}

Term $(B)$ is repeatedly separated into four terms:
\begin{gather}
\underbrace{\left|\mathbb{E}_{\mathbf{P}}\left[V_{t+1}^{\pi}\left(b_{t}az\right)\mid b_{t}\right]-\frac{1}{C}\sum_{i=1}^{C}\hat{\tilde{V}}_{\omega,t+1}^{\pi}\left(\bar{b}_{t+1}^{\prime\left[I_{i}\right]}\right)\right|}_{(B)}\\
\leq\underbrace{\left|\mathbb{E}_{\mathbf{P}}\left[V_{t+1}^{\pi}\left(b_{t}az\right)\mid b_{t}\right]-\frac{\sum_{i=1}^{C}w_{t}^{i}\boldsymbol{V}_{t+1}^{\pi}\left(b_{t},a\right)^{\left[i\right]}}{\sum_{i=1}^{C}w_{t}^{i}}\right|}_{(1)\text{ Importance sampling error}}\\
+\underbrace{\left|\frac{\sum_{i=1}^{C}w_{t}^{i}\boldsymbol{V}_{t+1}^{\pi}\left(b_{t},a\right)^{\left[i\right]}}{\sum_{i=1}^{C}w_{t}^{i}}-\frac{1}{C}\sum_{i=1}^{C}\boldsymbol{V}_{t+1}^{\pi}\left(b_{t},a\right)^{\left[I_{i}\right]}\right|}_{(2)\text{ MC weighted sum approximation error}}\\
+\underbrace{\left|\frac{1}{C}\sum_{i=1}^{C}\boldsymbol{V}_{t+1}^{\pi}\left(b_{t},a\right)^{\left[I_{i}\right]}-\frac{1}{C}\sum_{i=1}^{C}V_{t+1}^{\pi}\left(b_{t}az^{\left[I_{i}\right]}\right)\right|}_{(3)\text{ MC next-step integral approximation error}}\\
+\underbrace{\left|\frac{1}{C}\sum_{i=1}^{C}V_{t+1}^{\pi}\left(b_{t}az^{\left[I_{i}\right]}\right)-\frac{1}{C}\sum_{i=1}^{C}\hat{\tilde{V}}_{\omega,t+1}^{\pi}\left(\bar{b}_{t+1}^{\prime\left[I_{i}\right]}\right)\right|}_{(4)\text{Function estimation error}}\\
\leq\frac{1}{4}\lambda+\frac{1}{4}\lambda+\frac{1}{2}\lambda+\alpha_{t+1}
\end{gather}

Note the notations used:
\begin{gather}
Z_{t+1}\triangleq p_{Z}\left(z\mid x_{t+1}\right)\\
T_{t,t+1}\triangleq p_{T}\left(x_{t+1}\mid x_{t},a\right)\\
T_{t,t+1}^{\left[i\right]}\triangleq p_{T}\left(x_{t+1}\mid x_{t}^{i},a\right)
\end{gather}
For the rest of this proof, we define $V_{t}^{\pi}\left(b_{t}\right)$ to be the value function attained 
from time $t$ given initial belief $b_{t}$ and by following policy 
$\pi$. Additionally, we define the following:
\begin{gather}
\boldsymbol{V}_{t+1}^{\pi}\left(b_{t},a,x_{t}\right)\bydef \\
\intop_{\mathcal{X}}\intop_{\mathcal{Z}}V_{t+1}^{\pi}\left(b_{t}az\right)\left(Z_{t+1}\right)\left(T_{t,t+1}\right)\dif x_{t+1}\dif z\\
\boldsymbol{V}_{t+1}^{\pi}\left(b_{t},a\right)^{\left[i\right]}\triangleq\boldsymbol{V}_{t+1}^{\pi}\left(b_{t},a,x_{t,i}\right)=\\
\intop_{\mathcal{X}}\intop_{\mathcal{Z}}V_{t+1}^{\pi}\left(b_{t}az\right)Z_{t+1}T_{t,t+1}^{\left[i\right]}\dif x_{t+1}\dif z \\
\mathbb{E}_{\mathbf{P}}\left[V_{t+1}^{\pi}\left(b_{t}az\right)\mid b_{t}\right]\bydef\\
\intop_{\mathcal{X}}\intop_{\mathcal{X}}\intop_{\mathcal{Z}}V_{t+1}^{\pi}\left(b_{t}az\right)\left(Z_{t+1}\right)\left(T_{t,t+1}\right)p\left(x_{t}\mid b_{t}\right)\dif x_{t:t+1}\dif z\\
=\intop_{\mathcal{X}}\boldsymbol{V}_{t+1}^{\pi}\left(b_{t},a\right)p\left(x_{t}\mid b_{t}\right)\dif x_{t}\\
=\frac{\intop_{\mathcal{X}^{t+1}}\boldsymbol{V}_{t+1}^{\pi}\left(b_{t},a\right)\left(Z_{1:t}\right)\left(T_{1:t}\right)p\left(x_{0}\mid b_{0}\right)\dif x_{0:t}}{\intop_{\mathcal{X}^{t+1}}\left(Z_{1:t}\right)\left(T_{1:t}\right)p\left(x_{0}\mid b_{0}\right)\dif x_{0:t}}
\end{gather}
\paragraph{(1) Importance sampling error}

This term is the difference between the conditional expectation $\mathbb{E}_{\mathbf{P}}\left[V_{t+1}^{\pi}\left(b_{t}az\right)\mid b_{t}\right]$
and its SN estimator. We have $\left\Vert \boldsymbol{V}_{t+1}^{\pi}\right\Vert _{\infty}\leq V_{\max}$,
therefore we can apply the SN concentration inequality from Lemma \ref{lem:SNBounds} to
bound it by the augmented $\lambda\slash4$:
\begin{gather}
    \begin{split}
        \probd\left(\left| 
            \vphantom{\frac{\sum_{i=1}^{C}w_{t}^{i}\boldsymbol{V}_{t+1}^{\pi}\left(b_{t},a\right)^{\left[i\right]}}{\sum_{i=1}^{C}w_{t}^{i}}} 
        \mathbb{E}_{\mathbf{P}}\left[V_{T+1}^{\pi}\left(b_{t}az\right)\mid b_{t}\right] \right. \right. \\
        \left. \left. -\frac{\sum_{i=1}^{C}w_{t}^{i}\boldsymbol{V}_{t+1}^{\pi}\left(b_{t},a\right)^{\left[i\right]}}{\sum_{i=1}^{C}w_{t}^{i}}\right|\leq\frac{\lambda}{4}\right)
    \end{split} \\
    \geq1-3\exp\left(-C\cdot k_{\max}^{2}\left(\lambda,C\right)\right)
\end{gather}

\paragraph{(2) Monte Carlo weighted sum approximation error}

First, we assume that all variables $\left\{ s_{t}^{i},w_{t}^{i}\right\} ,b_{t},a$
are given, and only $I$ is random. Note that $\left\Vert \boldsymbol{V}_{t+1}^{\pi}\left(b_{t},a,\cdot\right)\right\Vert _{\infty}\leq V_{\max}$.

We will define the discrete probability mass defined by the weights
at depth $t$: $p_{w,t}\left(I=i\right)\triangleq\left(w_{t}^{i}\slash\sum_{j}w_{t}^{j}\right)$,
and for convenience denote $\boldsymbol{V}\left(i\right)\triangleq\boldsymbol{V}_{t+1}^{\pi}\left(b_{t},a\right)^{\left[i\right]}$.
The term $\frac{\sum_{i=1}^{C}w_{t}^{i}\boldsymbol{V}_{t+1}^{\pi}\left(b_{t},a\right)^{\left[i\right]}}{\sum_{i=1}^{C}w_{t}^{i}}$
is equivalent to the expectation of $\boldsymbol{V}\left(I\right)$
w.r.t. $p_{w,t}\left(I=i\right)$. The term $\frac{1}{C}\sum_{i=1}^{C}\boldsymbol{V}_{t+1}^{\pi}\left(b_{t},a\right)^{\left[I_{i}\right]}$
is equivalent to a Monte Carlo average of the previous quantity with
$C$ samples. Therefore:
\begin{gather}
\left|\frac{\sum_{i=1}^{C}w_{t}^{i}\boldsymbol{V}_{t+1}^{\pi}\left(b_{t},a\right)^{\left[i\right]}}{\sum_{i=1}^{C}w_{t}^{i}}-\frac{1}{C}\sum_{i=1}^{C}\boldsymbol{V}_{t+1}^{\pi}\left(b_{t},a\right)^{\left[I_{i}\right]}\right|\\
\Rightarrow\left|\mathbb{E}_{p_{w,t}\left(I=i\right)}\left[\boldsymbol{V}\left(I\right)\right]-\frac{1}{C}\sum_{i=1}^{C}\boldsymbol{V}\left(I_{i}\right)\right|
\end{gather}

This is the form of the double-sided Hoeffding-type bound on the function
values $\boldsymbol{V}\left(I\right)$. Hence, we can choose $\lambda$
such that for an arbitrary fixed set $\left\{ s_{t}^{i},w_{t}^{i}\right\} ,b_{t},a$:
\begin{gather}
    \begin{split}
        \probd\left(\left|\mathbb{E}_{p_{w,t}\left(I=i\right)}\left[\boldsymbol{V}\left(I\right)\right]-\frac{1}{C}\sum_{i=1}^{C}\boldsymbol{V}\left(I_{i}\right)\right| \right. \\
        \left. \vphantom{\sum_{i=1}^{C}}
             \leq\lambda\mid\left\{ s_{t}^{i},w_{t}^{i}\right\} ,b_{t},a\right)
    \end{split} \\
    \geq1-2\exp\left(-C\lambda^{2}\slash 2V_{\max}^{2}\right)
\end{gather}

We will use the two following well-known facts:
\begin{enumerate}
\item The probability of an event $A$ is equal to the expectation of the
indicator $\boldsymbol{1}_{A}$, i.e. $\probd\left(A\right)=\mathbb{E}\left(\boldsymbol{1}_{A}\right)$.
\item The tower property, also known as the law of total expectation: for
any two random variables $X,Y$ defined on the same probability space,
holds that $\mathbb{E}\left[X\right]=\mathbb{E}\left[\mathbb{E}\left[X\mid Y\right]\right]$.
\end{enumerate}
Therefore
\begin{gather}
\probd\left(\left|\mathbb{E}_{p_{w,d}\left(I=i\right)}\left[\boldsymbol{V}\left(I\right)\right]-\frac{1}{C}\sum_{i=1}^{C}\boldsymbol{V}\left(I_{i}\right)\right|\leq\lambda\right)\\
\overset{\text{Expt. of indicator}}{=}\mathbb{E}\left[\boldsymbol{1}_{\probd\left(\left|\mathbb{E}_{p_{w,d}\left(I=i\right)}\left[\boldsymbol{V}\left(I\right)\right]-\frac{1}{C}\sum_{i=1}^{C}\boldsymbol{V}\left(I_{i}\right)\right|\leq\lambda\right)}\right]\\
\begin{split}
\overset{\text{Tower property}}{=}\mathbb{E}\left[\mathbb{E}\left[\boldsymbol{1}_{\probd\left(\left|\mathbb{E}_{p_{w,d}\left(I=i\right)}\left[\boldsymbol{V}\left(I\right)\right]-\frac{1}{C}\sum_{i=1}^{C}\boldsymbol{V}\left(I_{i}\right)\right|\leq\lambda\right)}\mid \right.\right.\\
\left\{ s_{d,i},w_{d,i}\right\} ,b_{d},a\bigg]\bigg]
\end{split} \\
\begin{split}
\overset{\text{Expt. of indicator}}{=}\mathbb{E}\left[\probd\left(\left|\mathbb{E}_{p_{w,d}\left(I=i\right)}\left[\boldsymbol{V}\left(I\right)\right]-\frac{1}{C}\sum_{i=1}^{C}\boldsymbol{V}\left(I_{i}\right)\right| \right. \right. \\
\leq\lambda\mid\left\{ s_{d,i},w_{d,i}\right\} ,b_{d},a\bigg)\bigg]
\end{split} \\
\geq\mathbb{E}\left[1-2\exp\left(-C\lambda^{2}\slash2V_{\max}^{2}\right)\right]\\
=1-2\exp\left(-C\lambda^{2}\slash2V_{\max}^{2}\right)
\end{gather}

We choose to bound term $(2)$ with augmented $\lambda\slash4$, and
this holds with probability at least $1-2\exp\left(-C\lambda^{2}\slash32V_{\max}^{2}\right)$:

\begin{gather}
    \begin{split}
    \probd\left(\left|\frac{\sum_{i=1}^{C}w_{t}^{i}\boldsymbol{V}_{t+1}^{\pi}\left(b_{t},a\right)^{\left[i\right]}}{\sum_{i=1}^{C}w_{t}^{i}}-\right.\right. \\
    \frac{1}{C}\sum_{i=1}^{C}\boldsymbol{V}_{t+1}^{\pi}\left(b_{t},a\right)^{\left[I_{i}\right]}\bigg|\leq\frac{\lambda}{4}\bigg)
    \end{split}\\
\geq1-2\exp\left(-C\lambda^{2}\slash32V_{\max}^{2}\right)
\end{gather}

\paragraph{(3) Monte Carlo next-step integral approximation error}

First we define $\Delta_{t+1}\left(b_{t},a\right)^{\left[I_{i}\right]}\triangleq\boldsymbol{V}_{t+1}^{\pi}\left(b_{t},a\right)^{\left[I_{i}\right]}-V_{t+1}^{\pi}\left(b_{t}az^{\left[I_{i}\right]}\right)$.
Note that by rearranging the summation we can write 
\begin{gather}
\left|\frac{1}{C}\sum_{i=1}^{C}\boldsymbol{V}_{t+1}^{\pi}\left(b_{t},a\right)^{\left[I_{i}\right]}-\frac{1}{C}\sum_{i=1}^{C}V_{t+1}^{\pi}\left(b_{t}az^{\left[I_{i}\right]}\right)\right|\\
=\left|\frac{1}{C}\sum_{i=1}^{C}\Delta_{t+1}\left(b_{t},a\right)^{\left[I_{i}\right]}\right|
\end{gather}

Note that $V_{t+1}^{\pi}\left(b_{t}az^{\left[I_{i}\right]}\right)$
is simply a single sample Monte Carlo approximation of $\boldsymbol{V}_{t+1}^{\pi}\left(b_{t},a\right)^{\left[I_{i}\right]}$,
as the random vector $\left(x_{t+1,I_{i}},z_{I_{i}}\right)$ is jointly
generated using the generative model according to the correct probability
$Z_{t+1}\cdot T_{t,t+1}$ given $x_{t}^{i}$. This can be seen by the
following:
\begin{gather}
\mathbb{E}_{\left(s_{d+1,I_{i}},o_{I_{i}}\right)}\left[V_{d+1}^{\pi}\left(b_{d}ao^{\left[I_{i}\right]}\right)\right]\\
=\intop_{\mathcal{X}}\intop_{\mathcal{O}}V_{d+1}^{\pi}\left(b_{d}ao\right)Z_{d+1}T_{d,d+1}^{\left[i\right]}\dif s_{d+1}\dif z\\
=\boldsymbol{V}_{d+1}^{\pi}\left(b_{d},a\right)^{\left[I_{i}\right]}
\end{gather}
Hence follows from the tower property that $\mathbb{E}\left[\Delta_{t+1}\right]=0$.
Define $\Delta_{t+1}\left(b_{t},a\right)^{\left[I_{i}\right]}\triangleq\boldsymbol{V}_{t+1}^{\pi}\left(b_{t},a\right)^{\left[I_{i}\right]}-V_{t+1}^{\pi}\left(b_{t}az^{\left[I_{i}\right]}\right)$.
Note from the triangle inequality that $\left\Vert \Delta_{t+1}\right\Vert _{\infty}\leq2V_{\max}$.
Since $I_{i}$ are i.i.d. follows that $\Delta_{t+1}\left(b_{t},a\right)^{\left[I_{i}\right]}$
are i.i.d. too, and we can directly use another Hoeffding's bound:
\begin{gather}
\probd\left(\left|\frac{1}{C}\sum_{i=1}^{C}\Delta_{t+1}\left(b_{t},a\right)^{\left[I_{i}\right]}\right|\leq\frac{\lambda}{2}\right)\\
=\probd\left(\left|\frac{1}{C}\sum_{i=1}^{C}\Delta_{t+1}\left(b_{t},a\right)^{\left[I_{i}\right]}-\mathbb{E}\left[\Delta_{t+1}\right]\right|\leq\frac{\lambda}{2}\right)\\
\geq 1 - 2\exp\left(-C\lambda^{2}\slash32V_{\max}^{2}\right)
\end{gather}

\paragraph{(4) Function estimation error}

From the inductive hypothesis, for each possible $az^{\left[I_{i}\right]}$,
we have the bound $\left|Q_{\mathbf{P},t}^{\pi}\left(b_{t}az^{\left[I_{i}\right]},a^{\prime}\right)-\hat{\tilde{Q}}_{\omega,t}^{\pi}\left(\bar{b}_{t}az^{\left[I_{i}\right]},a^{\prime}\right)\right|\leq\alpha_{t+1}$
holding with high probability for all actions $a^{\prime}$, hence
in particular for $a^{\prime}=\pi\left(b_{t}az^{\left[I_{i}\right]}\right)$.
Then follows:
\begin{gather}
\left|\frac{1}{C}\sum_{i=1}^{C}V_{t+1}^{\pi}\left(b_{t}az^{\left[I_{i}\right]}\right)-\frac{1}{C}\sum_{i=1}^{C}\hat{\tilde{V}}_{\omega,t+1}^{\pi}\left(\bar{b}_{t+1}^{\prime\left[I_{i}\right]}\right)\right|\\
\leq\frac{1}{C}\sum_{i=1}^{C}\left|V_{t+1}^{\pi}\left(b_{t}az^{\left[I_{i}\right]}\right)-\hat{\tilde{V}}_{\omega,t+1}^{\pi}\left(\bar{b}_{t+1}^{\prime\left[I_{i}\right]}\right)\right|\\
\leq\frac{1}{C}\sum_{i=1}^{C}\alpha_{t+1}=\alpha_{t+1}
\end{gather}

\paragraph{Combining the bounds}

Thus, each of the terms are bound by $(A)\leq\frac{R_{\max}}{4V_{\max}}\lambda$
and $(B)\leq\frac{1}{4}\lambda+\frac{1}{4}\lambda+\frac{1}{2}\lambda+\alpha_{t+1}$,
which uses the SN concentration bounds twice and Hoeffding's bound
twice. Combining together:
\begin{gather}
\left|Q_{\mathbf{P},t}^{\pi}\left(b_{t},a\right)-\hat{\tilde{Q}}_{\omega,t}^{\pi}\left(\bar{b}_{t},a\right)\right|\\
\leq\frac{R_{\max}}{4V_{\max}}\lambda+\nu+\gamma\left[\frac{1}{4}\lambda+\frac{1}{4}\lambda+\frac{1}{2}\lambda+\alpha_{t+1}\right]\\
\leq\frac{1-\gamma}{4}\lambda+\nu+\gamma\lambda+\gamma\alpha_{t+1}\\
\leq\lambda+\nu+\gamma\alpha_{t+1}=\alpha_{t}
\end{gather}

We obtain the worst case union bound on the probability that all inequalities
simultaneously hold. We used the SN concentration bound twice, Hoeffding's
bound twice, and assumed reward bound once. For the SN and Hoeffding's
bounds, we can bound the worst case probability of either by the following
\begin{gather}
\max\left\{ 3\exp\left(-C\cdot k_{\max}^{2}\left(\lambda,C\right)\right),2\exp\left(-C\lambda^{2}\slash32V_{\max}^{2}\right)\right\} \\
\leq3\exp\left(-C\cdot\acute{k}^{2}\right),
\end{gather}
which we multiply by the union factor bound $\left(4C\right)^{L+1}$
since we want the function estimates to be within the bounds for the
specific action chosen by the policy, and all child nodes (used $C$
times in the function estimation error), and we used either SN concentration
bound or Hoeffding's bound 4 times in total. For the reward approximation
bound, we multiply by a factor of $C^{L+1}$ to account for the branching
factor. Hence, in total, we have shown that for all levels $t$ the
worst case union bound probability of all bad events is bounded by
\begin{gather}
\probd\left(\left|Q_{\mathbf{P},t}^{\pi}\left(b_{t},a\right)-\hat{\tilde{Q}}_{\omega,t}^{\pi}\left(\bar{b}_{t},a\right)\right|\leq\alpha_{t}\right)\\
\geq1-3\left(4C\right)^{L+1}\left(\exp\left(-C\cdot\acute{k}^{2}\right)\right)-C^{L+1}\alpha\left(\nu,N\right)\\
\geq1-3\left(4C\right)^{L+1}\left(\exp\left(-C\cdot\acute{k}^{2}\right)+\alpha\left(\nu,N\right)\right)
\end{gather}

When giving a union bound for all possible actions as well (i.e. for
all policies) then the action branching factor becomes $\left|A\right|^{L+1}$,
once from bounding the root level, and another $L$ times for all
depth levels. Thus, in this case the probability becomes at least $1-3\left(4\left|A\right|C\right)^{L+1}\left(\exp\left(-C\cdot\acute{k}^{2}\right)+\alpha\left(\nu,N\right)\right)$.

\paragraph{PB-MDP Value Convergence}

Similarily to the previous convergence bound, we split the difference
into two terms, the reward estimation error $(A)$ and the next-step
value estimation error $(B)$:

\begin{gather}
\left|Q_{\mathbf{M_{P}},t}^{\pi}\left(\bar{b}_{t},a\right)-\hat{\tilde{Q}}_{\omega,t}^{\pi}\left(\bar{b}_{t},a\right)\right|\\
\leq\underbrace{\left|\rho\left(\bar{b}_{t},a\right)-\tilde{\rho}\left(\bar{b}_{t},a\right)\right|}_{(\mathrm{A})}\\
\underbrace{\begin{split}
    +\gamma \bigg|\mathbb{E}_{\mathbf{M}_{\mathbf{P}}}\left[V_{\mathbf{M}_{\mathbf{P}},t+1}^{\pi}\left(\bar{b}_{t+1}\right)\mid\bar{b}_{t},a\right] \\
    -\frac{1}{C}\sum_{i=1}^{C}\hat{\tilde{V}}_{\omega,t+1}^{\pi}\left(\bar{b}_{t+1}^{\prime\left[I_{i}\right]}\right)\bigg|.
\end{split}}_{(\mathrm{B})}
\end{gather}

Term $(A)$ can be bounded like the reward approximation error of
term $(A)$ in the previous case.
\begin{equation}
\probd\left(\left|\frac{\sum_{i=1}^{C}w_{t}^{i}r_{t}^{i}}{\sum_{i=1}^{C}w_{t}^{i}}-\frac{\sum_{i=1}^{C}w_{t}^{i}\tilde{r}_{t}^{i}}{\sum_{i=1}^{C}w_{t}^{i}}\right|\leq\nu\right)\geq1-C\delta_r\left(\nu,N_r\right)
\end{equation}

For the inductive step, we prove that the difference $(B)$ is bounded
for all $t=0,\dots,L$. We split it into two terms:
\begin{gather}
\left|\mathbb{E}_{\mathbf{M}_{\mathbf{P}}}\left[V_{\mathbf{M}_{\mathbf{P}},t+1}^{\pi}\left(\bar{b}_{t+1}\right)\mid\bar{b}_{d},a\right]-\frac{1}{C}\sum_{i=1}^{C}\hat{\tilde{V}}_{\omega,t+1}^{\pi}\left(\bar{b}_{t+1}^{\prime\left[I_{i}\right]}\right)\right|\\
\leq\underbrace{
    \begin{split}
        \bigg|\mathbb{E}_{\mathbf{M}_{\mathbf{P}}}\left[V_{\mathbf{M}_{\mathbf{P}},t+1}^{\pi}\left(\bar{b}_{t+1}\right)\mid\bar{b}_{t},a\right]- \\
        \left. \frac{1}{C}\sum_{i=1}^{C}V_{\mathbf{M}_{\mathbf{P}},t+1}^{\pi}\left(\bar{b}_{t+1}^{\prime\left[I_{i}\right]}\right)\right|
    \end{split}
    }_{(1)\text{MC transition approximation error}}\\
+\underbrace{\left|\frac{1}{C}\sum_{i=1}^{C}V_{\mathbf{M}_{\mathbf{P}},t+1}^{\pi}\left(\bar{b}_{t+1}^{\prime\left[I_{i}\right]}\right)-\frac{1}{C}\sum_{i=1}^{C}\hat{\tilde{V}}_{\omega,t+1}^{\pi}\left(\bar{b}_{t+1}^{\prime\left[I_{i}\right]}\right)\right|}_{(2)\text{ Function approximation error}}\\
\leq\underbrace{\lambda}_{(1)}+\underbrace{\beta_{t+1}}_{(2)}.
\end{gather}

\paragraph{(1) MC transition approximation error}

This term is a Monte Carlo estimate of the integration over the transition
estimate $\tau\left(\bar{b}_{t+1}\mid\bar{b}_{t},a\right)$. The value
function and its estimate are both bounded by $V_{\max}$, therefore
we can invoke Hoeffding's bound to obtain the following probabilistic
bound:
\begin{gather}
\begin{split}
    \probd\bigg(\bigg|\mathbb{E}_{\mathbf{M}_{\mathbf{P}}}\left[V_{\mathbf{M}_{\mathbf{P}},t+1}^{\pi}\left(\bar{b}_{t+1}\right)\mid\bar{b}_{t},a\right]- \\
    \left. \left. \frac{1}{C}\sum_{i=1}^{C}V_{\mathbf{M}_{\mathbf{P}},t+1}^{\pi}\left(\bar{b}_{t+1}^{\left[I_{i}\right]}\right)\right|\leq\lambda\right)\\
    \geq1-2\exp\left(-C\lambda^{2}\slash2V_{\max}^{2}\right). 
\end{split}
\end{gather}

\paragraph{(2) Function approximation error}

From the inductive hypothesis, for each $\bar{b}_{t+1}^{\prime\left[I_{i}\right]}$,
its PB-MDP $Q$-value it's sparse sampling-$\omega$-$\pi$ estimate
at step $t+1$ is bounded by $\beta_{t+1}$ for all actions. In particular,
this also applies for $a=\pi\left(\bar{b}_{t+1}^{\prime\left[I_{i}\right]}\right)$.
Thus follows 
\begin{gather}
    \left|\frac{1}{C}\sum_{i=1}^{C}V_{\mathbf{M}_{\mathbf{P}},t+1}^{\pi}\left(\bar{b}_{t+1}^{\prime\left[I_{i}\right]}\right)-\frac{1}{C}\sum_{i=1}^{C}\hat{\tilde{V}}_{\omega,t+1}^{\pi}\left(\bar{b}_{t+1}^{\prime\left[I_{i}\right]}\right)\right|\\
    \leq\frac{1}{C}\sum_{i=1}^{C}\left|V_{\mathbf{M}_{\mathbf{P}},t+1}^{\pi}\left(\bar{b}_{t+1}^{\prime\left[I_{i}\right]}\right)-\hat{\tilde{V}}_{\omega,t+1}^{\pi}\left(\bar{b}_{t+1}^{\prime\left[I_{i}\right]}\right)\right|\\
    \leq\frac{1}{C}\sum_{i=1}^{C}\beta_{d+1}=\beta_{t+1}
\end{gather}

\paragraph{Combining the bounds}

By applying similar logic of ensuring that every particle belief node
and satisfies the concentration inequalities, we combine one Hoeffding's
inequality with one immediate reward approximation error. The terms
are bounded by $(A)\leq\nu$ and $(B)\leq\lambda+\beta_{t+1}$, therefore:
\begin{gather}
\left|Q_{\mathbf{M_{P}},t}^{\pi}\left(\bar{b}_{t},a\right)-\hat{\tilde{Q}}_{\omega,t}^{\pi}\left(\bar{b}_{t},a\right)\right| \\
\leq\nu+\gamma\left(\lambda+\beta_{t+1}\right)=\beta_{t},
\end{gather}

and with the union bound, we get the following probabilistic bound
\begin{gather}
\probd\left(\left|Q_{\mathbf{M_{P}},t}^{\pi}\left(\bar{b}_{t},a\right)-\hat{\tilde{Q}}_{\omega,t}^{\pi}\left(\bar{b}_{t},a\right)\right|\leq\beta_{t}\right)\\
\geq1-2\cdotp C^{L+1}\left(\exp\left(-C\lambda^{2}\slash2V_{\max}^{2}\right)\right)-C^{L+1}\delta_r\left(\nu,N_r\right)\\
\geq1-2\cdotp C^{L+1}\left(\exp\left(-C\lambda^{2}\slash2V_{\max}^{2}\right)+\delta_r\left(\nu,N_r\right)\right).
\end{gather}
For the case of bounding for all policies simultaneously, the
probabilities become at least $1-2\cdotp\left(\left|A\right|C\right)^{L+1}\left(\exp\left(-C\lambda^{2}\slash2V_{\max}^{2}\right)+\delta_r\left(\nu,N_r\right)\right)$

\paragraph{Combining Both Concentration Bounds}

In order to enable the two concentration inequalities to simultaneously
hold, we bound the worst case union probability:
\begin{gather}
3\left(4C\right)^{L+1}\left(\exp\left(-C\cdot\acute{k}^{2}\right)+\delta_r\left(\nu,N_r\right)\right) \\+2\cdotp C^{L+1}\left(\exp\left(-C\lambda^{2}\slash2V_{\max}^{2}\right)+\delta_r\left(\nu,N_r\right)\right)\\
\leq3\left(4C\right)^{L+1}\left(\exp\left(-C\cdot\acute{k}^{2}\right)+\delta_r\left(\nu,N_r\right)\right)\\+2\left(4C\right)^{L+1}\left(\exp\left(-C\cdot\acute{k}^{2}\right)+\delta_r\left(\nu,N_r\right)\right)\\
=5\left(4C\right)^{L+1}\left(\exp\left(-C\cdot\acute{k}^{2}\right)+\delta_r\left(\nu,N_r\right)\right)
\end{gather}
Therefore, we conclude that the Sparse Sampling-$\omega$-$\pi$ $Q$-value
estimate concentration inequalities approximation error, for both
the original POMDP and its PB-MDP approximation, are bounded by $\alpha_{t}$,
$\beta_{t}$ at every belief node, respectively, with probability
at least $1-5\left(4C\right)^{L+1}\left(\exp\left(-C\cdot\acute{k}^{2}\right)+\delta_r\left(\nu,N_r\right)\right)$.
If we require the concentration inequality to simultaneously hold
for all policies, then the probability becomes $1-5\left(4\left|A\right|C\right)^{L+1}\left(\exp\left(-C\cdot\acute{k}^{2}\right)+\delta_r\left(\nu,N_r\right)\right)$.

\end{proof}

\begin{theoremrpt}[Generalized PB-MDP Convergence] 
    \label{thm:ProbConvergeRpt}
    Assume that the immediate state reward estimate is probabilistically bounded such that
    $\OpProbFlat{\absvalflat{r_{i}^{j}-\tilde{r}_{i}^{j}}\geq\nu}\leq\delta_{r}(\nu,N_r)$, for a number of reward samples $N_r$ and state sample $x_{i}^{j}$. 
    Assume that $\delta_{r}(\nu,N_r)\to0$ as $N_r\to\infty$.
    For all policies $\pi$, $t=0,\dots,L$ and $a\in\mathcal{A}$, the following bounds hold with probability of at least $1-5(4C)^{L+1}(\exp(-C\cdot\acute{k}^{2})+\delta_{r}(\nu,N_r))$:
    \begin{gather}
        \textstyle
        \absvalflat{Q_{\OrigPOMDP,t}^{\pi,[{p_Z}\slash{q_Z}]}(b_{t},a)-Q_{\PBMDP,t}^{\pi,[{p_Z}\slash{q_Z}]}(\PB{t}, a)}\leq \alpha_{t}+\beta_{t},
    \end{gather}
    where,
    \begin{gather}
        \textstyle
        \alpha_{t}=(1+\gamma)\lambda+\gamma \alpha_{t+1},\ \alpha_{L}=\lambda\geq0, \\
        \beta_{t}=2\nu+\gamma \beta_{t+1},\ \beta_{L}=2\nu\geq 0, \\
        k_{\max}(\lambda, C)=\frac{\lambda}{4V_{\max}d_{\infty}^{\max}}-\frac{1}{\sqrt{C}}>0,\\
        \acute{k}=\min\{k_{\max},\lambda\slash4\sqrt{2}V_{\max}\}.
    \end{gather}
    If we require the bound to hold for all possible policies that can be extracted from a given belief tree simultaneously, then under the assumption of a finite action space, the probability is at least $1-5(4\absvalflat{\mathcal{A}}C)^{L+1}(\exp(-C\cdot\acute{k}^{2})+\delta_{r}(\nu,N_r))$.
\end{theoremrpt}

\begin{proof}
    We prove for the observation model $p_Z$ without loss of generality.

    Under the same conditions and probability as Lemma \ref{lem:SparseSamplingOmegaPiConvergence}, we bound the difference directly between the theoretical action value function and the particle-belief approximation using the triangle inequality,
\begin{gather}
    \left|Q_{\mathbf{P},t}^{\pi}(b_{t},a)-Q_{\mathbf{M_{P}},t}^{\pi}(\PB{t},a)\right|\leq\alpha_{t}+\beta_{t}.
\end{gather}
We define the following recursive bounds
\begin{gather}
    A_{t}\bydef (1+\gamma)\lambda+\gamma A_{d+1},\ A_{L}\bydef \lambda \\
    B_{t}\bydef 2\nu+\gamma B_{t+1},\ B_{L}\bydef 2\nu
\end{gather}
and follows that $A_{t}+B_{t}=\alpha_{t}+\beta_{t}$ and $A_{L}+B_{L}=\alpha_{L}+\beta_{L}$, hence follows that $\left|Q_{\mathbf{P},t}^{\pi}(b_{t},a)-Q_{\mathbf{M_{P}},t}^{\pi}(\PB{t},a)\right|\leq A_{t}+B_{t}$. By renaming $\alpha_t \bydef A_t$ and $\beta_t \bydef B_t$ we get the required result.
\end{proof}
    
\subsection{Corollary \ref{crl:ArbitraryPrecisionBoundsRpt}}

\begin{corollaryrpt}
    \label{crl:ArbitraryPrecisionBoundsRpt}
    For arbitrary precision $\varepsilon$ and accuracy $\delta$ we can choose constants $\lambda, \nu, C, N_r$ such that the following holds with probability of at least $1-\delta$:
    \begin{gather}
        \absvalflat{Q_{\OrigPOMDP,t}^{\Policy{},[{p_Z}\slash{q_Z}]}(\PB{t},a)-Q_{\PBMDP,t}^{\Policy{},[{p_Z}\slash{q_Z}]}(\PB{t}, a)}\leq \varepsilon.
    \end{gather}
\end{corollaryrpt}

\begin{proof}
    We prove for the case of bounding for a belief tree with a finite action space. 
    The case of a single policy can be proven similarly by removing the factors related $\absval{\mathcal{A}}$.

    Let $\varepsilon>0$ and let $\lambda>0$.
    We denote $L_{+1} \bydef L+1$.

    Let $\nu,\lambda=\frac{1}{4L_{+1}}\varepsilon$.

    The conditions necessary for Theorem \ref{thm:ProbConvergeRpt} are the following:
    \begin{gather}
        k_{\max}(\lambda, C)=\frac{\lambda}{4V_{\max}d_{\infty}^{\max}}-\frac{1}{\sqrt{C}}>0, \label{eq:condition1} \\
        \delta \geq 5(4\absvalflat{\mathcal{A}}C)^{L_{+1}}(\exp(-C\cdot\acute{k}^{2})+\delta_{r}(\nu,N_r)) \label{eq:condition2} \\
        \acute{k}=\min\{k_{\max},\lambda\slash4\sqrt{2}V_{\max}\}.
    \end{gather}

    Denote $A_1\bydef \frac{\lambda}{4V_{\max}d_{\infty}^{\max}}$, $A_2\bydef \lambda\slash4\sqrt{2}V_{\max}$, and note $A_1, A_2>0$.
    We obtain $k_{\max}(\lambda, C) = A_1 - \frac{1}{\sqrt{C}}$.
    
    We would like to choose a particle count $C$ large enough such that $k_{\max}(\lambda, C)$ is larger than $A_1\slash 2$.
    Hence we denote the solution for the following equation with $C_{A_1\slash2}$:
    \begin{gather}
        (A_1 - \frac{1}{\sqrt{C_{A_1\slash2}}})^2\cdot C_{A_1\slash2} = \frac{A_1}{2} \\
        \Rightarrow C_{A_1\slash2}\bydef \frac{2\sqrt{2}\sqrt{\frac{1}{(A_1)^{3}}}(A_1)^{2}+ A_1 + 2}{2 (A_1)^2}. 
    \end{gather}
    Denote the following constants:
    \begin{gather}
        K_1 \bydef \max\{C_{A_1\slash2},\frac{1}{(A_1)^2}, \frac{2}{(A_2)^2}\} \\
        K_2 \bydef \min \{\frac{A_1}{2}, 2\} > 0, \\
        K_3 \bydef 5(4\absvalflat{\mathcal{A}})^{L_{+1}}
    \end{gather}

    We choose an auxiliary particle count $\tilde{C}\in \mathbb{N}$, such that the particle count $C$ satisfies
    \begin{gather}
        C>K_1 \cdot \tilde{C} \geq K_1.
    \end{gather}
    Condition \eqref{eq:condition1} is satisfied because 
    \begin{gather}
        k_{\max}(\lambda, C) = A_1 - \frac{1}{\sqrt{C}} > A_1 - \frac{1}{\sqrt{\frac{1}{(A_1)^2}}}, \\
        A_1 - \sqrt{(A_1)^2} = 0.
    \end{gather}
    Additionally, 
    \begin{gather}
        k_{\max}(\lambda, C)^2 \cdot C > \frac{A_1}{2}, \\
        (A_2)^2 \cdot C > (A_2)^2 \cdot \frac{2}{(A_2)^2} = 2.
    \end{gather}
    We obtain that
    \begin{gather}
        \acute{k}^2 \cdot C = \\
        \min \{k_{\max}(\lambda, C)^2, (A_2)^2\} \cdot C > \\
        \min \{\frac{A_1}{2}, 2\} \cdot \tilde{C} = K_2 \cdot \tilde{C}.
    \end{gather}
    Therefore, condition \eqref{eq:condition2} will be satisfied if the following is satisfied:
    \begin{gather}
        \delta \geq K_3 (K_1)^{L_{+1}} \cdot \tilde{C}^{L_{+1}} (\exp(-K_2 \cdot \tilde{C}) + \delta_{r}(\nu,N_r)).
    \end{gather}
    We perform the change of variables $X \bydef K_2\cdot \tilde{C}$, i.e. $\tilde{C} = \frac{X}{K_2}$:
    \begin{gather}
        K_3 (K_1)^{L_{+1}} \cdot \tilde{C}^{L_{+1}} \exp(-K_2 \cdot \tilde{C}) \\
        = K_3(\frac{K_1}{K_2})^{L_{+1}}\cdot X^{L_{+1}} \exp(-X).
    \end{gather}
    The exponential function grows faster than any polynomial, and specifically for $P(X)=K_3(\frac{K_1}{K_2})^{L_{+1}}\cdot X^{L_{+1}}$. Therefore, we can choose $X^\prime \in \mathbb{R}$ such that $\forall X\in\mathbb{N}>X^{\prime}$:
    \begin{gather}
        K_3(\frac{K_1}{K_2})^{L_{+1}}\cdot X^{L_{+1}} \exp(-X) \leq \frac{\delta}{2}. \label{eq:result_1}
    \end{gather}
    By choosing an auxiliary particle count $\tilde{C} > \frac{X^\prime}{K_2}$ we satisfy \eqref{eq:result_1}. 

    For the choice of $\lambda$ the following holds:
    \begin{gather}
        \alpha_0 \leq \sum_{k=0}^{L} 2\lambda \cdot \gamma^k \leq 2\lambda L_{+1} = \frac{2\cdotp \varepsilon(L_{+1})}{4L_{+1}}=\frac{\varepsilon}{2}.
    \end{gather}
    For the given choice of the particle count $C$, we remind the assumption in Theorem \ref{thm:ProbConvergeRpt} that for all $\nu>0$ it holds that $\delta_r(\nu,N_r)\to 0$ as $N_r\to\infty$. Therefore, we can choose $N_r^{\prime}$ such that the following holds:
    \begin{gather}
        K_3 C^{L_{+1}} \cdot \delta_r(\nu,N_r) \leq \frac{\delta}{2}
    \end{gather}
    For the choice of $\nu$ the following holds:
    \begin{gather}
        \beta_0 \leq \sum_{k=0}^{L} 2\nu \cdot \gamma^k \leq 2\nu L_{+1} = \frac{2\cdotp \varepsilon (L_{+1})}{4L_{+1}}=\frac{\varepsilon}{2}.
    \end{gather}
    In summary, for the choices of:
    \begin{gather}
        C>\max \{K_1, \frac{K_1}{K_2}X^\prime\} \\
        N_r > N_r^{\prime}
    \end{gather}
    we have that the following statements hold,
    \begin{gather}
        K_3 C^{L_{+1}}\cdot\exp(-C\cdot\acute{k}^{2}) \leq \frac{\delta}{2} \\
        K_3 C^{L_{+1}} \cdot \delta_r(\nu,N_r) \leq \frac{\delta}{2},
    \end{gather}
    And therefore from Theorem \ref{thm:ProbConvergeRpt}, with probability of at least $1-K_3C^{L_{+1}}(\exp(-C\cdot\acute{k}^{2}) + \delta_r(\nu,N_r))\geq 1-\delta$ the following bound holds:
    \begin{gather}
        \left|Q_{\mathbf{P},t}^{\pi}(b_{t},a)-Q_{\mathbf{M_{P}},t}^{\pi}(\PB{t},a)\right|\leq\alpha_{t}+\beta_{t} \\
        \leq \alpha_0 + \beta_0 \leq \frac{\varepsilon}{2} + \frac{\varepsilon}{2} = \varepsilon.
    \end{gather}
\end{proof}

\subsection{Corollary \ref{crl:JointApproximationBoundRpt}}

\begin{corollaryrpt} 
    \label{crl:JointApproximationBoundRpt}
    Assuming that $\mathcal{P}$ is an MDP planner that can approximate $Q$-values with arbitrary precision $\varepsilon^{\mathcal{P}}$ at an accuracy $1-\delta^{\mathcal{P}}$, we denote the precision and accuracy of the action value and action cumulative bound functions:
    \begin{gather}
        \OpProbFlat{\absvalflat{Q_{\PBMDP,t}^{\pi,q_Z}(\bar{b}_{t},a)-\Estim{Q}_{\PBMDP,t}^{\pi,q_Z}(\bar{b}_{t},a)}\leq\varepsilon^{\mathcal{P}}_{Q}}\geq 1-\delta^{\mathcal{P}}_{Q} \label{eq:PlannerBoundsQ_rpt}\\
        \OpProbFlat{\absvalflat{\Phi_{\PBMDP,t}^{\pi}(\bar{b}_{t},a)-\PhiImSaEst_{\PBMDP,t}^{\pi}(\bar{b}_{t},a)}\leq\varepsilon^{\mathcal{P}}_{\Phi}}\geq 1-\delta^{\mathcal{P}}_{\Phi} \label{eq:PlannerBoundsPhi_rpt}
    \end{gather}
    From Corollary \ref{crl:ArbitraryPrecisionBounds} it holds that we can choose constants $\lambda, \nu, C, N_r$ such that the following holds,
    \begin{gather}
        \OpProbFlat{\absvalflat{Q_{\OrigPOMDP,t}^{\pi,q_Z}(b_{t},a)-Q_{\PBMDP,t}^{\pi,q_Z}(\PB{t}, a)}\leq \varepsilon_{Q}}\geq 1-\delta_{Q} \label{eq:PBMDPBoundsQ_rpt}, \\
        \OpProbFlat{\absvalflat{\Phi_{\OrigPOMDP,t}^{\pi}(b_{t},a)-\RewardEstim{\Phi}_{\PBMDP,d}^{\pi}(\PB{t}, a)}\leq \varepsilon_{\Phi}}\geq 1-\delta_{\Phi}. \label{eq:PBMDPBoundsPhi_rpt}
    \end{gather}
    Then with probability of at least $1-(\delta_{Q}+\delta^{\mathcal{P}}_{Q}+\delta_{\Phi}+\delta^{\mathcal{P}}_{\Phi})$
    \begin{gather}
        \absvalflat{Q_{\OrigPOMDP,t}^{\pi,p_Z}(b_{t},a)-\Estim{Q}_{\PBMDP,t}^{\pi,q_Z}(\bar{b}_{t},a)}\leq \\
        \PhiImSaEst_{\PBMDP,t}^{\pi}(\bar{b}_{t},a) + \varepsilon_{Q} + \varepsilon^{\mathcal{P}}_{Q} + \varepsilon_{\Phi} + \varepsilon^{\mathcal{P}}_{\Phi}.
    \end{gather}
\end{corollaryrpt}

\begin{proof}
    We combine all probabilistic bounds with the triangle inequality and the union bound, to conclude that with probability of at least $1-(\delta_{Q}+\delta^{\mathcal{P}}_{Q}+\delta_{\Phi}+\delta^{\mathcal{P}}_{\Phi})$:
    \begin{gather}
        \absvalflat{Q_{\OrigPOMDP,t}^{\pi,p_Z}(b_{t},a)-\Estim{Q}_{\PBMDP,t}^{\pi,q_Z}(\PB{t},a)}\\
        \leq \absvalflat{Q_{\OrigPOMDP,t}^{\pi,p_Z}(b_{t},a) - Q_{\OrigPOMDP,t}^{\pi,q_Z}(b_{t},a)} \, (\textit{Thm. \ref{thm:LocalActionBoundRpt}}) \\
        + \absvalflat{Q_{\OrigPOMDP,t}^{\pi,q_Z}(b_{t},a) - Q_{\PBMDP,t}^{\pi,q_Z}(\PB{t},a)} \,(\textit{eq. \eqref{eq:PBMDPBoundsQ_rpt}}) \\
        + \absvalflat{Q_{\PBMDP,t}^{\pi,q_Z}(\PB{t},a) -\Estim{Q}_{\PBMDP,t}^{\pi,q_Z}(\PB{t},a)} \, (\textit{eq. \eqref{eq:PlannerBoundsQ_rpt}}) \\
        \leq \Phi_{\OrigPOMDP,t}^{\pi}(b_{t},a) + \varepsilon_{Q} + \varepsilon^{\mathcal{P}}_{Q} \quad(\textit{eq. \eqref{eq:PBMDPBoundsPhi_rpt} + eq. \eqref{eq:PlannerBoundsPhi_rpt}}) \\
        \leq \PhiImSaEst_{\PBMDP,t}^{\pi}(\PB{t},a) + \varepsilon_{Q} + \varepsilon^{\mathcal{P}}_{Q} + \varepsilon_{\Phi} + \varepsilon^{\mathcal{P}}_{\Phi}.
    \end{gather}
\end{proof}

\section{Further Implementation Details}

\subsection{Simulative Setting}

In the 2D beacons environment, the state and observation spaces are defined as the whole plane: $\mathcal{X}=\mathcal{Z}=\mathbb{R}^2$. The planning horizon is $L=15$.

We define the 2D rectangle $Rect((x_1,y_1),(x_2,y_2))$ as the axis-aligned rectangle starting from the bottom-left corner $(x_1,y_1)$ to the top-right corner $(x_2,y_2)$:
\begin{gather}
    Rect((x_1,y_1),(x_2,y_2)) \bydef \\
    \left\{ (x,y)\in\mathbb{R}^2 \mid x_1 \leq x \leq x_2, y_1 \leq y\leq y_2\right\}
\end{gather}
The outer walls are defined by $\mathcal{X}_\textit{collision}\bydef \neg Rect((-2, 0),(12,6))$.
The goal region is defined by $\mathcal{X}_\textit{goal}=Rect ((4,-1.5),(6, 0))$. 
The prior is the following Gaussian mixture with 2 components:
\begin{gather}
    b_0(x_0) = \frac{1}{2}\mathcal{N}(\mu_1, \Sigma_0) + \frac{1}{2}\mathcal{N}(\mu_2, \Sigma_0), \\
    \mu_1 = (1,2), \mu_2=(9,2), \\
    \Sigma_0 = \operatorname{diag}(\sigma_{x_0}^2,\sigma_{y_0}^2), \\
    \sigma_{x_0} = 0.5, \sigma_{y_0}=0.25.
\end{gather}
The action space is the discrete space of the following 4 actions: $\mathcal{A}=\{(0,1), (0,-1), (1,0), (-1,0)\}$.
The transition model is the following Gaussian model:
\begin{gather}
    p_T(x^\prime\mid x, a) = \mathcal{N}(x+a,\Sigma_T), \\
    \Sigma_T = \operatorname{diag}(\sigma_T^2, \sigma_T^2) , \\
    \sigma_T^2 = 0.15.
\end{gather}
In the arena the 6 beacons are arranged in a horizontal row in equal distances along the 2D line from $(0, 4)$ to $(10, 4)$, such that their locations are:
\begin{gather}
    (x_1^\textit{beacon}, y_1^\textit{beacon}) = (0,4), \\ 
    (x_2^\textit{beacon}, y_2^\textit{beacon}) = (2,4), \\
    (x_3^\textit{beacon}, y_3^\textit{beacon}) = (4,4), \\
    (x_4^\textit{beacon}, y_4^\textit{beacon}) = (6,4), \\
    (x_5^\textit{beacon}, y_5^\textit{beacon}) = (8,4), \\
    (x_6^\textit{beacon}, y_6^\textit{beacon}) = (10,4).
\end{gather}
The range of a beacon is $R^\textit{beacon}=1$.
The light and dark regions are defined as the following:
\begin{gather}
    \mathcal{X}_\textit{light} = \bigcup_{i=1}^6{B((x_i^\textit{beacon}, y_i^\textit{beacon}); R^\textit{beacon})} \\
    \mathcal{X}_\textit{dark} = \mathcal{X} \setminus \mathcal{X}_\textit{light}
\end{gather}

The original and simplified observation models in the dark region are the same Gaussian model:
\begin{gather}
    p_Z(z\mid x) = q_Z(z\mid x) = \mathcal{N}(x, \Sigma_\textit{dark}), \\
    \Sigma_\textit{dark} = \operatorname{diag}(\sigma_\textit{dark}^2, \sigma_\textit{dark}^2), \\
    \sigma_\textit{dark} = 5.
\end{gather}
The original model in the light region is defined as a Gaussian mixture model made of rings of components with increasing number of components but decaying weights, centered around a center component. Its definition is the following:
\begin{gather}
    N_\sigma = 3, k_r = 10, k_\theta = 25 \\
    \sigma_\textit{light} = 0.3, \\
    \Sigma_{p_Z} = \operatorname{diag}(\sigma_\textit{light}\frac{N_\sigma}{k_r}) \\
    N_i^\theta = \max\{1,i\cdotp k_\theta\}, \\
    r_i = N_\sigma\cdotp i, \\
    w_i = \exp(-\frac{(i \frac{N_\sigma}{k_r})^2}{2}), \\
    \tilde{w}_i = w_i \slash (\sum_{k=1}^{k_r} \sum_{j=1}^{N_i^\theta} w_k), \\
    v_{i, j}^{\theta} = i \frac{N_\sigma}{k_r} Rot(\frac{2\pi j}{N_i^\theta}) \begin{bmatrix}
        \sigma_\textit{light} \\
        \sigma_\textit{light}
    \end{bmatrix} \\
    p_Z(z\mid x \in \mathcal{X}_\textit{light}) = \sum_{i=1}^{k_r} \sum_{j=1}^{N_i^\theta} \tilde{w}_i \mathcal{N}(x+\cdotp v_{i, j}^{\theta}, \Sigma_{p_Z})
\end{gather}
where $Rot(\varphi)$ is the 2D rotation matrix corresponding to angle $\varphi$.
The simplified observation model is the following single component Gaussian:
\begin{gather}
    q_Z(z\mid x \in \mathcal{X}_\textit{light}) = \mathcal{N}(x, \Sigma_\textit{light}), \\
    \Sigma_\textit{light} = \operatorname{diag}(\sigma_\textit{light}^2, \sigma_\textit{light}^2), \\
    \sigma_\textit{light} = 0.3.
\end{gather}

The reward function is time and state dependent only, and defined as the sum of 3 indicators:
\begin{gather}
    r_{t}(x)=R_{\textit{hit}}\cdot\boldsymbol{1}_{x\in\mathcal{X}_{\textit{goal}}}\\
    + R_{\textit{miss}}\cdot\boldsymbol{1}_{x\notin\mathcal{X}_{\textit{goal}}} \\
    + R_{\textit{collide}}\cdot\boldsymbol{1}_{x\in\mathcal{X}_{\textit{collision}}}
\end{gather}
For all time steps $t > 0$ we set $R_{\textit{hit}}=100$, $R_{\textit{collide}}=-50$. We set $R_{\textit{miss}}=-50$ if $t=L$, and otherwise $-1$ if $t>0$. In the first time step the reward is 0, i.e. $r_0=0$. Additionally, there is no discount factor, i.e. $\gamma=1$.
From these definitions, it holds that $\Rmax{i}=100$ and $\Vmax{t}=100 + (15-t)$ for $1\leq t\leq15$.

\subsection{Bound Estimate}

\begin{algorithm}[tb]
    \caption{Empirical TV-Distance $\DZEst$}
    \label{alg:DZEst}
    \textbf{Algorithm:} $\text{Estimate }\DZ$.\\
    \textbf{Input}: Number of state samples $N_\Delta$, number of observation per state sample $N_Z$, total variation threshold $\Delta_\textit{Thresh}$, $Q_0$ state proposal distribution, $p_Z$ original measurement model, $q_Z$ simplified observation model. \\
    \textbf{Output}: Number of kept states $N_\Delta^{\textit{eff}}$, a set of states $\{x_n^{\Delta}\}_{n=1}^{N_\Delta^{\textit{eff}}}$, TV-distance estimates $\{ \DZEst(x_n^{\Delta}) \}_{n=1}^{N_\Delta^{\textit{eff}}}$.
    \begin{algorithmic}[1]
    \STATE $X=list(), D=list()$
    \FORALL {{$n=1,\dots, N_\Delta$}}
    \STATE $x_n^{\Delta} \sim Q_0$
    \FORALL {$j=1,\dots,N_Z$}
    \STATE $z_n^j \sim (p_Z + q_Z) \slash 2$
    \ENDFOR
    \STATE $\DZEst(x_{n}^{\Delta}) = \sum_{j=1}^{N_Z}{2\cdot\frac{\lvert p_Z(z_{j}^{n}\mid x_{n}^{\Delta}) - q_Z(z_{j}^{n}\mid x_{n}^{\Delta}) \rvert}{p_Z(z_{j}^{n}\mid x_{n}^{\Delta}) + q_Z(z_{j}^{n}\mid x_{n}^{\Delta})}}$
    \IF {$\DZEst(x_{n}^{\Delta}) > \Delta_\textit{Thresh}$}
    \STATE $X\leftarrow x_{n}^{\Delta}, D \leftarrow \DZEst(x_{n}^{\Delta})$
    \ENDIF
    \ENDFOR
    \STATE \textbf{return} $(\#X), X, D$
    \end{algorithmic}
\end{algorithm}

\begin{algorithm}[tb]
    \caption{Empirical Bounds $\MImSa_i$}
    \label{alg:EmpiricalBounds}
    \textbf{Algorithm:} $\text{Estimate }m_i\left(\bar{b},a,i\right)$.\\
    \textbf{Parameters}: Number of state samples $N_X$. \\
    \textbf{Input}: Particle belief set $\bar{b}_i=\left\{ \left(x_{i},w_{i}\right)\right\}$, action $a$, time step $i$. \\
    \textbf{Output}: A scalar $\Estim{\MImSa}_i(\bar{b}_i,a)$ that is an estimate of $m_i(\bar{b}_i,a)$.
    \begin{algorithmic}[1]
    \FORALL {$j=1,\dots,N_X$}
    \STATE $x_i^j \sim \bar{b}_i$
    \ENDFOR
    \STATE \textbf{return} $\frac{1}{N_x}\sum_{j=1}^{N_X}{\text{Estimate }m_i\left(x_i^j,a,i\right)}$
    \end{algorithmic}
    
    \textbf{Algorithm:} $\text{Estimate }m_i\left(x,a,i\right)$.\\
    \textbf{Parameters}: Transition model $p_T$, Distance threshold $d_T$, $Q_0$ state proposal distribution. \\
    \textbf{Input}: State $x_{i}$, action $a$, time step $i$. \\
    \textbf{Output}: A scalar $\MImSa_i(x_i,a)$ that is an estimate of $m_i(x_i,a)$.
    \begin{algorithmic}[1]
    \STATE $Neighborhood = list()$ \\
    \COMMENT{Neighborhood search can be implemented with KD-Tree}
    \FORALL {$n=1,\dots,N_{\Delta}^{\textit{eff}}$} 
    \IF {$\lVert x_i - x_n^\Delta \rVert \leq d_T$} 
    \STATE $Neighborhood \leftarrow x_n^\Delta$
    \ENDIF
    \ENDFOR
    \STATE $m\leftarrow 0$
    \FORALL {$x_n^\Delta \in Neighborhood$}
    \STATE $m \leftarrow m + \Vmax{i+1}\cdot \frac{p_T(x_n^\Delta \mid x_i, a)}{N_\Delta Q_0(x_n^\Delta)} \DZEst(x_n^\Delta)$
    \ENDFOR
    \STATE \textbf{return} $m$
    \end{algorithmic}
\end{algorithm}

We now describe the parameters describing the computation of $\DZEst$ and $\MImSa_i$.
In Algorithm \ref{alg:DZEst} we describe the process of sampling the delta states and estimating $\DZEst$. In Algorithm \ref{alg:EmpiricalBounds} we describe the procedures for computing $\MImSa_i$ for a particle belief and for a state sample.

We chose the number of delta states $N_\Delta=2048$, the number of observation samples for estimating $\DZEst$ is $N_{Z}=256$.
The number of particles in the belief $C=250$, and $N_X=30$ is the number of particles used for computing $\Estim{\MImSa}_{i}$.
We chose $\Delta_{\textit{Thresh}}=10^{-4}$, the threshold for filtering delta states with low delta. 
The threshold distance for the transition model is $d_T=0.6=4\cdotp\sigma_T$.
After pre-filtering of $\DZEst(x_{n}^{\Delta}) > \Delta_{\textit{Thresh}}$, we had $N_{\Delta}^{\textit{eff}}=251$, sample mean of $\Estim{\mathbb{E}}[\DZEst]=8.31\cdotp10^{-2}$, and sample variance $\Estim{\sigma}[\DZEst]=0.67\cdotp10^{-2}$.
Minimum and maximum values are $\min \DZEst=7.06\cdotp10^{-2}$, $\max \DZEst=12.08\cdotp10^{-2}$.

\subsection{Solver}

We now describe details regarding the implementation of the PFT-DPW solver \cite{Sunberg18icaps}.

We implement the particle filter belief with 2 key points to note. The first is that we use resampling for stability of the particle filter, as when the particle filter algorithm was running without resampling we ran into particle depletion issues in a few scenarios. The second is that during planning with a particle belief, since our scenario may end prematurely, we often get a situation where only some particles terminate and others do not. To solve this problem, we note that in general, we may describe the reward of a belief with the law of total expectation:
\begin{gather}
    Q_t^\pi(b_t, a) = r_t(b_t, a) + \gamma \ExptFlat{z_{t+1}}{}{V_{t+1}^\pi(b_{t+1})} \\
    r_t(b_t, a) \\
    + \gamma \ExptFlat{z_{t+1}}{}{V_{t+1}^\pi(b_{t+1})\mid T(b_{t+1})} \probd(T(b_{t+1})) \\
    + \gamma \ExptFlat{z_{t+1}}{}{V_{t+1}^\pi(b_{t+1})\mid \neg T(b_{t+1})} \probd(\neg T(b_{t+1}))
\end{gather}
where $T(b_{t+1})$ is a random variable indicating that a belief has terminated, whether because of states being in terminal states (due to entering goal or collision states), or due to reaching the time limit. The future value for all terminated states is $0$, therefore the action value is equal to:
\begin{gather}
    Q_t^\pi(b_t, a) = r_t(b_t, a) \\
    + \gamma \ExptFlat{z_{t+1}}{}{V_{t+1}^\pi(b_{t+1})\mid \neg T(b_{t+1})} \probd(\neg T(b_{t+1}))
\end{gather}
We implemented $\probd(\neg T(b_{t+1}))$ as summing the total particle weights that are in goal or collision states after sampling from the transition model.
To implement $\ExptFlat{z_{t+1}}{}{V_{t+1}^\pi(b_{t+1})\mid \neg T(b_{t+1})}$, we added a parameter controlling the normalized sum of the weights of a particle belief, which we multiply after each transition by $\probd(\neg T(b_{t+1}))$.
This conditions all future beliefs of the same branch by the same factor, and trickles down the belief tree recursively.

The rollout policy is based on maximum likelihood transitions and observations.
It steers the empirical mean of the belief towards the goal by executing each time step the action that has maximal inner product with the goal's center.

We implement in PFT-DPW a secondary reward that is stored in all posterior belief nodes.
We modified the posterior nodes in the planner's belief tree to include $\PhiImSaEst_{\PBMDP,t}$ values in addition to $\Estim{Q}$. 
In addition, we hooked on the generative model of the environment to additionally compute $m_i$ in the simplified planning, such that $(\PB{i+1},r_{i+1},m_i)\sim G(\PB{i},a_i)$.

The parameters of PFT-DPW were taken to be the following:
\begin{table}[H]
    \normalsize
    \centering
    \begin{tabularx}{0.95\columnwidth}
        {| >{\centering\arraybackslash}X 
        | >{\centering\arraybackslash}X 
        | >{\centering\arraybackslash}X | }
        \hline
        Parameter & Notation (PFT-DPW) & Value \\
        \hline
        No. of simulations & $n$ & 500 \\
        \hline
        Exploration bonus & $c$ & 50 \\
        \hline
        Action prog. widening mult. const. & $k_a$ & 1.1 \\
        \hline
        Action prog. widening exp. const. & $\alpha_a$ & 0.24 \\
        \hline
        Obs. prog. widening mult. const. & $k_o$ & 1.1 \\
        \hline
        Obs. prog. widening exp. const. & $\alpha_o$ & 0.19 \\
        \hline
    \end{tabularx}
\end{table}
Additionally, the number of particles in the particle filter was $C=250$.
} 
{}

\end{document}